\def\year{2018}\relax
\documentclass[letterpaper]{article}
\usepackage{aaai18}
\usepackage{times}
\usepackage{helvet}
\usepackage{courier}
\usepackage{dsfont}
\usepackage{bbm}

\usepackage{amssymb,amsmath,amsthm}
\usepackage{xspace}
\usepackage{amsthm}
\usepackage{amsmath}
\usepackage{amssymb}
\usepackage{mathrsfs}

\usepackage{tikz}
\usetikzlibrary{positioning}
\usetikzlibrary{decorations.text}
\usetikzlibrary{arrows}
\usetikzlibrary{shapes}

\usepackage{algorithm,algorithmicx,algpseudocode}

\usepackage{amsmath,amssymb,mathtools,subfigure}

\usepackage{url}
\usepackage{xcolor}
\newcommand{\red}[1]{#1}

\newcommand{\blue}[1]{#1}

\newcommand{\cA}[0]{\mathcal{A}}
\newcommand{\cC}[0]{\mathcal{C}}
\newcommand{\cS}[0]{\mathcal{S}}
\newcommand{\expect}[1]{\mathbb{E}\left[ #1 \right]}

\newcommand{\alice}[0]{\textit{alice}}
\newcommand{\bob}[0]{\textit{bob}}
\newcommand{\eve}[0]{\textit{eve}}
\newcommand{\sm}[0]{\textit{sm}}
\newcommand{\fr}[0]{\textit{fr}}

\newtheorem{theorem}{Theorem}
\newtheorem{lemma}{Lemma}
\newtheorem{proposition}[theorem]{Proposition}
\newtheorem{definition}{Definition}
\newtheorem{observation}[theorem]{Remark}
\newtheorem{example}[theorem]{Example}

\newcommand\myeq{\stackrel{\mathclap{\normalfont\mbox{\small{def}}}}{=}}

\newif\ifappendix
\appendixtrue

\newif\ifmaintext
\maintexttrue

\allowdisplaybreaks

\begin{document}
%
\title{Relational Marginal Problems: Theory and Estimation}
\author{Ond\v{r}ej Ku\v{z}elka\\
Cardiff University, UK\\
KuzelkaO@cardiff.ac.uk \\
 \And Yuyi Wang\\
ETH Zurich, Switzerland\\
yuwang@ethz.ch\\
\And
Jesse Davis\\
KU Leuven, Belgium\\
jesse.davis@cs.kuleuven.be\\
\And
Steven Schockaert\\
Cardiff University, UK\\
SchockaertS1@cardiff.ac.uk
}
\maketitle

\ifmaintext

\begin{abstract}
In the propositional setting, the {\em marginal problem} is to find a (maximum-entropy) distribution that has some given marginals. We study this problem in a relational setting and make the following contributions. First, we compare two different notions of relational marginals. Second, we show a duality between the resulting relational marginal problems and the maximum likelihood estimation of the parameters of relational models, which generalizes a well-known duality from the propositional setting. Third, by exploiting the relational marginal formulation, we present a statistically sound method to learn the parameters of relational models that will be applied in settings where the number of constants differs between the training and test data. Furthermore, based on a relational generalization of {\em marginal polytopes}, we characterize cases where the standard estimators based on feature's number of true groundings needs to be adjusted and we quantitatively characterize the consequences of these adjustments. Fourth, we prove bounds on expected errors of the estimated parameters, which allows us to lower-bound, among other things, the effective sample size of relational training data.
\end{abstract}

\section{Introduction}

Statistical Relational Learning (SRL, Getoor and Taskar, eds., \citeyear{srl}) is concerned with learning probabilistic models from relational data.
Many popular SRL frameworks, such as Markov Logic Networks (MLNs, Richardson and Domingos \citeyear{Richardson2006}), use weighted logical formulas to encode statistical regularities that hold for the considered problem. 
Typically, the maximum (pseudo-)likelihood weights of the formulas are estimated from training data, which is usually a single large example (e.g.\ a social network). 
This is problematic for two reasons. First, the weights that are learned from this single training example are in general not optimal for examples of different sizes \blue{\cite{JainKB07}}. This turns out to be a fundamental problem, which cannot simply be solved by rescaling the weights \cite{shalizi2013consistency}. Second, without making further assumptions, it is difficult to provide any statistical guarantees about the learned weights.


In this paper, we approach parameter estimation in SRL from a novel direction, by introducing the notion of a {\em relational marginal problem}. In the propositional case \cite{wainwright}, marginal problems entail finding a maximum-entropy distribution which has the given marginal probabilities. A well-known property of such problems is that they are dual to the maximum-likelihood estimation of the parameters of an undirected graphical model (where ``dual'' is in the sense of convex optimization). In relational marginal problems, we are similarly looking for a maximum-entropy distribution which satisfies some given statistics -- {\em relational marginals}. However, we also need to define what these relational marginals are. Thus, first, we describe two different types of {\em relational marginals}, which differ in the kinds of statistics that are provided.\footnote{\blue{All statistics in this paper are based on universally quantified formulas; we do not consider formulas with existential quantifiers.}} The first type is based on relational marginal distributions \cite{kuzelka.ijcai.17} and the second is based on Halpern-style random substitution semantics \cite{bacchus_halpern_koller}. Second, for both types of statistics, we establish a relational counterpart of the duality between maximum-likelihood estimation and max-entropy marginal problems. Interestingly, for the latter model, the corresponding dual is MLNs. 

Third, the relational marginal perspective allows us to learn parameters for domains that have different sizes (i.e., number of constants) than the training data. The basic idea to achieve this is simple. We assume that the training data is a sample of the data that we want to model.
For example, imagine trying to model all of Facebook based on a sampled subset of Facebook users along with all relations among them. 
Assuming the sample is a large enough and was obtained in a suitable way (which is not always the case in practice -- we discuss this issue later), the parameters of the marginals estimated from the sample should be close to the respective parameters for the whole network. Then, instead of using a model learned by optimizing the likelihood on the training data, we use a model obtained as a solution of the corresponding relational marginal problem with a domain of the required size. We may end up with estimated parameters for which the relational marginal problem has no solution. Therefore, we propose a method for adjusting the estimated parameters that enables a solution and characterize its effect on the estimates. Then we also \blue{introduce} the {\em relational marginal polytopes}, which allows us to provide conditions under which the unbiased unadjusted estimate will be valid (``realizable'') for domains of any size.

In addition, the relational marginal view of the  parameter learning problem, can be thought of as consisting of two decoupled problems: estimation of the  parameters of marginals and optimization to obtain the max-entropy distributions. Thus, to better understand parameter learning from relational data, it is important to characterize how accurate the estimates are. Assuming that all subsamples of the data being modeled are sampled with the same probability, we derive bounds on expected error, that is, the expected difference between the parameters obtained from the subsample and parameters that could be theoretically computed if the whole dataset were accessible (e.g. the whole Facebook). From this, we can also obtain lower-bounds on the effective sample size for relational data. 

The paper is structured around addressing the following four questions about relational marginal problems: 

\begin{enumerate}
    \item What should the relational marginals be? (Section {\em Two Types of Relational Marginals})
    \item What are the max-entropy distributions with the given relational marginals, and how can we find them? (Section {\em Max-Entropy Models})
    \item When are relational marginal problems realizable, and how can we adjust them when they are not? How can we adjust learning to account for differences in domain sizes? (Section {\em Realizability})
    \item How accurate are the parameter estimates of relational marginals, and what are the links with realizability? (Sections {\em Relational Marginal Polytopes} and {\em Estimation})
\end{enumerate}

\noindent {\em Proofs of all propositions stated in the main text and details of the duality derivations are contained in the \blue{online appendix}.\footnote{\blue{{\tt https://arxiv.org/abs/1709.05825}}}}

\section{Preliminaries}

This paper considers a function-free first-order logic language $\mathcal{L}$, which is built from a set of constants $\textit{Const}$, variables $\textit{Var}$ and predicates $\textit{Rel} = \bigcup_i \textit{Rel}_i$, where $\textit{Rel}_i$ contains the predicates of arity $i$. We assume an untyped language and use the domain size to refer to the $|Const|$. For $a_1,...,a_k \in \textit{Const}\cup \textit{Var}$ and $R \in \textit{Rel}_k$, we call $R(a_1,...,a_k)$ an atom.  If $a_1,..,a_k\in \textit{Const}$, this atom is called ground. A literal is an atom or its negation. 
We use $vars(\alpha)$ to denote the variables that appear in a formula $\alpha$.
The formula $\alpha_0$ is called a grounding of $\alpha$ if $\alpha_0$ can be obtained by replacing each variable in $\alpha$ with a constant from $\textit{Const}$. 
A formula is called closed if all variables are bound by a quantifier. 
A possible world $\omega$ is defined as a set of ground atoms. The satisfaction relation $\models$ is defined in the usual way.

A substitution is a mapping from variables to terms. An injective substitution is a substitution which does not map any two variables to the same term. As is commonly done in statistical relational learning, we use the unique names assumption, meaning that $c_1 \neq c_2$ whenever $c_1$ and $c_2$ are different constants. A first-order universally quantified formula $\alpha$ is said to be {\em proper} if $\alpha\vartheta$ is trivially \blue{true} whenever $\vartheta$ is not injective. \blue{For instance, the formula $\forall X,Y : \fr(X,Y)$ is not proper whereas the formula $\forall X,Y : \fr(X,Y) \vee X = Y$ is proper. 
We sometimes omit the universal quantifiers and simply write, e.g.\ $\fr(X,Y) \vee X = Y$}.

A (global) example is a pair $(\mathcal{A},\mathcal{C})$, with $\mathcal{C}$ a set of constants and $\mathcal{A}$ a set of ground atoms which only use constants from $\mathcal{C}$. Let $\Upsilon = (\mathcal{A},\mathcal{C})$ be an example and $\mathcal{S}\subseteq \mathcal{C}$. The fragment $\Upsilon\langle S \rangle = (\mathcal{B},\mathcal{S})$ is defined as the restriction of $\Upsilon$ to the constants in $\mathcal{S}$, i.e.\ $\mathcal{B}$ is the set of all atoms from $\mathcal{A}$ which only contain constants from $\mathcal{S}$. 
Two examples $\Upsilon_1 = (\mathcal{A}_1,\mathcal{C}_1)$ and $\Upsilon_2 = (\mathcal{A}_2,\mathcal{C}_2)$ are isomorphic, denoted as $\Upsilon_1 {\approx} \Upsilon_2$, if there exists a bijection $\sigma : \mathcal{C}_1 \rightarrow \mathcal{C}_2$ such that $\sigma(\mathcal{A}_1) = \mathcal{A}_2$, where $\sigma$ is extended to ground atoms in the usual way. When $\cC$ is a set of constants and $\Phi_0$ a set of closed formulas, $\Pi(\cC,\Phi_0)$ denotes the set of all $\Upsilon = (\cA,\cC)$ such that $\Upsilon \models \Phi_0$ (we can think of $\Phi_0$ as a set of constraints).

\blue{A Markov logic network (MLN, Richardson and Domingos \citeyear{Richardson2006}) is a set of weighted formulas $(\alpha,w)$, $w\in \mathbb{R}$ and $\alpha$ a function-free and quantifier-free first-order formula. The semantics are defined w.r.t.\ the groundings of the first-order formulas, relative to some finite set of constants $\mathcal{C}$, called the domain. An MLN is seen as a template that defines a Markov random field. Specifically, an MLN $\mathcal{M}$ induces the following probability distribution on the set of global examples\footnote{\blue{What we call a global example in this paper is usually called a {\em possible world} in the MLN literature.}} $\Upsilon$: 
\begin{eqnarray}
p_{\mathcal{M}}(\Upsilon) = \frac{1}{Z} \exp \left(\sum_{(\alpha,w) \in \mathcal{M}} w \cdot n(\alpha,\Upsilon)\right)
\label{e:mln}
\end{eqnarray}
\noindent
where $n(\alpha, \Upsilon)$ is the number of groundings of $\alpha$ that are satisfied in $\Upsilon$, and $Z$ is a normalization constant to ensure that $p_{\mathcal{M}}$ is a  probability distribution.}

\blue{Weights of MLNs are typically estimated using maximum (pseudo)likelihood from a given global example. When the MLN with weights learned in this way is used as a probabilistic model with a domain of different size, there are no guarantees regarding the induced distribution. This is most obvious when the MLN contains formulas having different numbers of variables. Then, keeping the weights fixed, the formulas with the highest number of variables often completely dominate the others if we increase the domain size. While some simple cases could be solved by normalizing the counts $n(\alpha,\Upsilon)$, in general this is not the case. Shalizi and Rinaldo (\citeyear{shalizi2013consistency}) list the example of modelling homophily in social networks; we refer to their paper for details.}

\section{Two Types of Relational Marginals}\label{sec:relmarginals}


Typically, parameters for a statistical relational model are estimated from a single example of a relational structure  that consists of a large set of ground atoms $\mathcal{A}$. Intuitively, the goal is to learn a probability distribution of such relational structure\blue{s}. The challenge is how to estimate the distribution from a single example. One solution is based on the assumption that the relational structure has repeated regularities. Then, statistics about these regularities can be computed for small substructures of the train example and used to construct a distribution over large relational structures. Thus, the next issue is how to construct the fragments and compute statistics on them. Next, we discuss two possible ways to do so, which we will refer to as Model A and Model B.



\subsection{Model A}
The first approach to constructing fragments is from~\cite{kuzelka.ijcai.17}. It repeatedly samples subsets $\mathcal{S} \subseteq \cC$ of the constants from the given example $\Upsilon = (\cA,\cC)$ and then builds one training example $\Upsilon \langle \mathcal{S} \rangle$ for each $\mathcal{S}$. However, the training examples must consider isomorphic classes of constants to account for the fact that each fragment will contain different constants. This is accomplished by using the notion of {\em local examples}. 





\begin{definition}[Local example]
Let $k\in \mathbb{N}$. A local example of width $k$ is a pair $\omega=(\mathcal{A},\{1,...,k\})$, where $\mathcal{A}$ is a set of ground atoms that contain only constants from the set $\{1,2,\dots,k\}$. For an example $\Upsilon = (\mathcal{A},\mathcal{C})$ and $\mathcal{S}\subseteq \mathcal{C}$, we write $\Upsilon[\mathcal{S}]$ for the set of all local examples of width $|\mathcal{S}|$ which are isomorphic to $\Upsilon\langle \mathcal{S} \rangle$. 
\end{definition}

\noindent To distinguish local examples from global examples, we will denote them using lower case Greek letters such as $\omega$ instead of upper case letters such as $\Upsilon$.

\begin{example}\label{ex1}
Let $\Upsilon = (\{ \fr(\alice,\bob),$ $\fr(\bob,\alice),$ $\fr(\bob,\allowbreak\eve),$ $\fr(\eve,\bob),$ $\sm(\alice) \},$ $\{\alice, \bob, \eve \})$, i.e.\ the only smoker is $\textit{alice}$ and the friendship structure is:
\begin{center}
\resizebox{0.215\textwidth}{!}{
\tikzset{
main node/.style={ellipse,fill=white!11,draw,minimum size=0.3cm,inner sep=0pt},
other node/.style={rectangle,fill=white!11,minimum size=0.3cm,inner sep=0pt},
}
\tikzset{edge/.style = {->,> = latex'}}
\begin{tikzpicture}

\node[main node] (1) {alice};
\node[main node] (2) [right = 0.5cm of 1] {bob};
\node[main node] (3) [right = 0.5cm of 2] {eve};

\draw[edge] (1) [bend right] to (2);
\draw[edge] (2) [bend right] to (1);
\draw[edge] (3) [bend right] to (2);
\draw[edge] (2) [bend right] to (3);
\end{tikzpicture}}
\end{center}
Then $\Upsilon\langle \{ \alice, \bob \} \rangle$ ${=}$ $(\{\sm(\alice),$ $\fr(\alice,\bob),$ $\fr(\bob,\alice) \},$ $\{\alice,\bob\})$, $\Upsilon[\{ \alice, \bob \}] {=}$ $\{ (\{ \fr(1,2),$ $\fr(2,1),\sm(1) \}, \{1,2\}), {\{} (\{ \fr(2,1), \fr(1,2), \sm(2) \},$ $\{1,2\}) \}.$

\end{example}

\noindent This leads to a natural definition of a probability distribution over local examples of width $k$.
\begin{definition}[Relational marginal distribution of a global example]
Let $\Upsilon = (\mathcal{A},\mathcal{C})$ be an example and $k\in \mathbb{N}$. The relational marginal distribution of $\Upsilon$ of width $k$ is a distribution $P_{\Upsilon,k}$ over local examples, where $P_{\Upsilon,k}(\omega)$ is defined as the probability that $\omega$ is sampled by the following process:
(i) Uniformly sample a subset $\mathcal{S}$ of $k$ constants from $\mathcal{C}$.
(ii) Uniformly sample a local example $\omega$ from the set $\Upsilon[\mathcal{S}]$.
For a closed formula $\alpha$ without constants, we also define, its probability:
$P_{\Upsilon,k}(\alpha) = \sum_{\omega : \omega \models \alpha} P_{\Upsilon,k}(\omega).$
\end{definition}

\noindent We will call a pair $(\alpha,p)$, where $\alpha$ is a constant-free closed formula and $p \in [0;1]$, a {\em relational marginal constraint}.
We may also interpret the probability $P_{\Upsilon,k}(\alpha)$ of a closed constant-free formula $\alpha$ as the probability that $\alpha$ is true in a restriction $\Upsilon\langle\mathcal{S}\rangle$ of $\Upsilon$ to a randomly sampled subset $\mathcal{S}$ of $k$ constants from $\Upsilon$. 
Thus, if  we are only interested in the probabilities of closed constant-free formulas, we do not have to refer to local examples. Local examples are important because relational marginal distributions defined using them are themselves probability distributions on possible worlds, which is both nice conceptually and convenient (as it means that we can model relational marginals of Model A using any standard propositional probabilistic model). \blue{Arguably, this property also makes Model A more interpretable and hence easier to explain to non-specialists.}


Global examples may also be assumed to be sampled from some distribution and we define the corresponding marginal distributions induced by such distributions accordingly. When $P(\Upsilon)$ is a distribution over finite global examples from a possibly countably infinite set $\Omega$, then the marginal distribution of width $k$ is a distribution $P_k$ over local examples where $P_k(\omega)$ is defined as
$P_k(\omega) = \sum_{\Upsilon \in \Omega} P(\Upsilon) \cdot P_{\Upsilon,k}(\omega).$
For a closed formula $\alpha$ without constants, we also analogically define:
$P_{k}(\alpha) = \sum_{\omega : \omega \models \alpha} P_{k}(\omega)$.
In other words, a relational marginal distribution is a mixture of (possibly countably many) relational marginal distributions of global examples.


\begin{proposition}\label{prop:unbiased_estimate}
Let $P(\Upsilon)$ be a distribution on domain size $n$ and $k \leq n$ be an integer. Let 
$\Omega_{\alpha} = \{ \Upsilon \langle \mathcal{S} \rangle : \mathcal{S} \subseteq \cC, |\mathcal{S}| = k, \Upsilon \langle \mathcal{S} \rangle \models  \alpha \}$ 
where $\Upsilon = (\cA, \cC)$ is sampled according to the distribution $P(\Upsilon)$. Then, for a closed \blue{constant-free} formula $\alpha$,
$\widehat{p}_{\alpha} = |\Omega_\alpha| \cdot \left( \begin{array}{c} n \\ k \end{array} \right)^{-1} $
is an unbiased estimate of $P_k(\alpha)$ for Model~A.
\end{proposition}

\subsection{Model B}
The second approach is to consider random substitutions, which is in the spirit of existing works \cite{bacchus_halpern_koller,schulte}. Here, the statistics that we collect about $\Upsilon$ are defined as follows.

\begin{definition}[Probability of formulas under Model B]\label{def:modelB}
Let $\Upsilon = (\cA,\cC)$ be a global example and $\alpha$ be a universally quantified formula. Let $P_\vartheta$ be a uniform distribution on injective substitutions from the set $\Theta_\alpha = \{ \vartheta | \vartheta : \textit{vars}(\alpha) \rightarrow \cC \mbox{ and } \vartheta \mbox{ is injective} \}$. Then the probability $Q(\alpha)$ of the formula $\alpha$ under model B is defined as
{\small $$Q_{\Upsilon}(\alpha) = \sum_{\vartheta \in \Theta_\alpha} \mathds{1}(\Upsilon \models \alpha\vartheta) P_{\vartheta}(\vartheta) = \frac{1}{|\Theta|}\sum_{\vartheta \in \Theta_\alpha} \mathds{1}(\Upsilon \models \alpha\vartheta).$$}
\end{definition}

\noindent Just like for Model A, we extend the definition of the probability of formulas straightforwardly to the case where $\Upsilon$ is not fixed but sampled from some distribution over a countable set $\Omega$: $Q(\alpha) = \sum_{\Upsilon \in \Omega} Q_\Upsilon(\alpha) \cdot P(\Upsilon)$.

\begin{example}
Let $\Upsilon$ be as in Example \ref{ex1}. Let 
\begin{align*}
\alpha = \forall X,Y : \neg \fr(X,Y) \vee \sm(Y),\\
\beta = \forall X,Y : \neg \fr(X,Y) \vee \sm(X) \vee \sm(Y).
\end{align*}
Assume that the relation $\fr(.,.)$, ``friends'', is symmetric. Then, the formula $\alpha$ is classically true if all people who are friends with someone, \blue{smoke}. The formula $\beta$ is classically true if for every pair of people $A$ and $B$ who are friends, at least one of them smokes. Computing the respective probabilities, we get $P_{\Upsilon,2}(\alpha) = \frac{1}{3}$, $P_{\Upsilon,2}(\beta) = \frac{2}{3}$, $Q_{\Upsilon}(\alpha) = \frac{1}{2}$, $Q_{\Upsilon}(\beta) = \frac{2}{3}$, which illustrates that in general the ``marginal'' probabilities given by the two models will differ. The first model might be slightly easier to interpret as the marginal probabilities $P(\gamma)$ correspond to the fraction of the width-$k$ fragments of $\Upsilon$ in which $\gamma$ is true as a classical logic formula.
\end{example}

We note that the straightforward analogue of Proposition~\ref{prop:unbiased_estimate} also holds for Model~B.

\section{Max-Entropy Models}

In this section we show how to compute {\em models} of given relational marginals under Model A and Model B.

\begin{definition}[Model of relational marginals]
Let us have a set of pairs $\Phi = \{ (\alpha_1,\theta_1),$ $\dots,$ $(\alpha_h,\theta_h)\}$ of relational marginals, where $\alpha_i$ is a closed formula and $\theta_i \in [0;1]$. Let $\Phi_0$ be a set of \blue{formulas, called} hard rules \blue{(i.e.\ formulas that are supposed to always hold)}. We say that a probability distribution $P(\Upsilon)$ over worlds satisfying the hard rules from $\Phi_0$ is a width-$k$ model of $\Phi$ iff $P_k(\alpha_i) = \theta_i$ ($Q_k(\alpha_i) = \theta_i$, respectively) for all $(\alpha_i,\theta_i) \in \Phi$.
\end{definition}

We will use standard duality arguments from convex optimization \cite{boyd2004convex}, essentially following \cite{optimization_entropy_counting}. \red{For both models, the result will in both cases an exponential-family model and the one for Model B will turn out to be equivalent to MLNs.}

\subsection{Model A}

Let $\cC$ be a set of constants, $\cA$ be the set of all atoms over $\cC$ based on some given first-order language and let $\mathcal{C}_k$ denote the set of all $k$-element subsets of $\cC$. Next, let $\Phi$ be a set of relational marginals and $\Phi_0$ be a set of hard constraints, i.e.\ formulas $\alpha$ such that if $\Upsilon \not\models \alpha$ then $P(\Upsilon) = 0$. We assume that there exists at least one distribution $P'$ which is a model of $\Phi$ and which satisfies the following {\em positivity} condition:  $P'(\Upsilon) > 0$ for all $\Upsilon$ satisfying the hard constraints (i.e.\ those for which $\Upsilon \models \Phi_0$).

\begin{equation}
    \sup_{\{ P_\Upsilon : \Upsilon \in \Pi(\cC,\Phi_0) \}}  \sum_{\Upsilon \in \Pi_{n}(\cC,\Phi_0)} P_{\Upsilon} \log{\frac{1}{P_\Upsilon}} \quad \textit{ s.t.}\label{eq:maximum_entropy_criterion}
\end{equation}
\begin{multline}\label{eq:maxent_marginal_constraints}
    \forall (\alpha_i,\theta_i) \in \Phi : \\
    \frac{1}{|\cC_k|} \sum_{\mathcal{S} \in \cC_k } \sum_{\Upsilon \in \Pi(\cC,\Phi_0)} \mathds{1}(\Upsilon\langle \mathcal{S} \rangle \models \alpha_i) \cdot P_\Upsilon = \theta_i
\end{multline}
\begin{align}
\forall \Upsilon \in \Pi_{n}(\cC,\Phi_0) : P_{\Upsilon} \geq 0, &&\sum_{\Upsilon \in \Pi(\cC,\Phi_0)} P_{\Upsilon} = 1\label{eq:maximum_entropy_cnormalization}
\end{align}

\noindent \red{Next we define $\#_k(\alpha,\Upsilon) = |\{ \mathcal{S} \in \mathcal{C}_k : \Upsilon\langle \mathcal{S} \rangle \models \alpha \}|$ as the number of sets $\mathcal{S} \in \cC_k$ such that the formula $\alpha$ is classically true in the restriction of $\Upsilon$ to constants in $\mathcal{S}$.} In the appendix, we show that the distribution, if it exists, that is a solution to the optimization problem has the following form:
\begin{equation}
    P_{\Upsilon} = \frac{1}{Z} \exp{\left(\sum_{\alpha_i \in \Phi} \lambda_i \frac{\#_k(\alpha_i,\Upsilon)}{|\cC_k|}\right)}.\label{eq:distr1}
\end{equation}

The Lagrangian dual problem of the maximum entropy problem is to maximize (where $\lambda_i \in \mathbb{R}$):
\begin{multline}
\sum_{\alpha_i \in \Phi} \lambda_i \theta_i - \log{\sum_{\Upsilon \in \Pi(\cC,\Phi_0)} e^{\sum_{\alpha_i \in \Phi} \lambda_i \frac{\#_k(\alpha_i,\Upsilon)}{|\cC_k|}}}\label{eq:dual}.
\end{multline}
Due to the positivity assumption, Slater's condition \cite{boyd2004convex} is satisfied and strong duality holds. Consequently, instead of solving the original problem, which has an intractable number of constraints and variables (one variable for each world $\Upsilon \in \Pi(\cC, \Phi_0)$), we can solve the dual problem, which only has $|\Phi|$ variables. On the other hand, the optimization criterion of the dual problem may still be computationally hard to solve as it requires weighted counting over worlds in $\Pi(\cC,\Phi_0)$. 
However, in many restricted, but non-trivial, cases, we can exploit lifted weighted model counting techniques in the same way as they were used for maximum-likelihood estimation in Van Haaren et al.\ (\citeyear{vanhaaren.mlj}).

Let us perform a change of variables $w_i := \lambda_i/|\cC_{k_i}|$. This gives us 
\begin{equation}
    P(\Upsilon) = \frac{1}{Z}\exp{\left(\sum_{\alpha_i \in \Phi} w_i \cdot \#_k(\alpha_i,\Upsilon)\right)}
\end{equation}
for the distribution and
\begin{multline}
\sum_{\alpha_i \in \Phi} w_i \theta_i |\cC_{k_i}| - \log{\sum_{\Upsilon \in \Pi(\cC,\Phi_0)} e^{\sum_{\alpha_i \in \Phi} w_i \cdot \#_k(\alpha_i,\Upsilon)}}\label{eq:dual2}
\end{multline}
for the optimization criterion of the dual problem. Assuming that the marginals were estimated from a global example $\widehat{\Upsilon} \in \Pi(\cC,\Phi_0)$ (note that here the domain $\cC$ is the same as the domain of the global examples over which the distribution is defined) and that they still satisfy the positivity assumption, we can also rewrite (\ref{eq:dual2}) as
{\small \begin{multline*}
\sum_{\alpha_i \in \Phi} w_i \cdot \#_k(\alpha_i,\widehat{\Upsilon}) - \log{\sum_{\Upsilon \in \Pi(\cC,\Phi_0)} e^{\sum_{\alpha_i \in \Phi} w_i \cdot \#_k(\alpha_i,\Upsilon)}}.
\end{multline*}}
It is straightforward to check that this is the same as directly optimizing the log-likelihood of $\widehat{\Upsilon}$. Thus, this is a relational analogue of the well-known duality of maximum likelihood and maximum entropy \cite{wainwright}.
Note that it is important for the duality of maximum likelihood and maximum entropy that both the $\widehat{\Upsilon}$, from which we estimated the parameters, and the global examples over which the distribution is computed have the same domain size. The Section {\em Realizability} will address cases where the domain sizes of the training and testing data differ.

\subsection{Model B}

Like for Model A, we can construct a convex optimization problem to obtain a maximum-entropy distribution with the given relational marginals under Model B. This problem's optimization criterion  is the same as (\ref{eq:maximum_entropy_criterion}). \red{To obtain the constraints enforcing the marginals, we can replace the summation over subsets of constants in $\cC$ in Equation~(\ref{eq:maxent_marginal_constraints}) by a summation over substitutions from $\Theta_{\alpha_i}$, where $\Theta_{\alpha_i}$ is defined as in Definition \ref{def:modelB}. This yields the following set of constraints for all $(\alpha_i,\theta_i) \in \Phi$:}
\begin{equation}\label{eq:maxent_marginal_constraintsB}
    \sum_{\vartheta \in \Theta_{\alpha_i}} \frac{1}{|\Theta_{\alpha_i}|} \sum_{\Upsilon \in \Pi(\cC,\Phi_0)} \mathds{1}(\Upsilon \models \alpha_i\vartheta) \cdot P_\Upsilon = \theta_i.
\end{equation}

Using basically the same reasoning as for Model A, we arrive at the following form of the probability distribution
\begin{equation}
    P(\Upsilon) = \frac{1}{Z}\exp{\left(\sum_{\alpha_i \in \Phi} w_i \cdot n(\alpha_i,\Upsilon)\right)}
\end{equation}
where $n(\alpha_i,\Upsilon)$ is the number of groundings $\alpha_i\vartheta$ of the formula $\alpha_i$, \blue{where} all $\vartheta$'s are injective, which are true in $\Upsilon$. 
\red{This distribution is identical to the one for MLNs which only contain proper formulas 
(because of the injectivity requirement in the definition of Model B).} The only difference with the distribution of Model A is the use of $n(\alpha_i,\Upsilon)$ instead of $\#_k(\alpha_i,\Upsilon)$. The dual optimization criterion for Model B then becomes to maximize
\begin{multline}
\sum_{\alpha_i \in \Phi} w_i \theta_i |\Theta_{\alpha_i}| - \log{\sum_{\Upsilon \in \Pi(\cC,\Phi_0)} e^{\sum_{\alpha_i \in \Phi} w_i \cdot n(\alpha_i,\Upsilon)}}\label{eq:dual2_mln}
\end{multline}
which can be rewritten, when $\theta_i$'s are estimated from some $\widehat{\Upsilon} = (\widehat{\cA},\widehat{\cC})$ with $|\widehat{\cC}| = |\widehat{\cC}|$, as $\sum_{\alpha_i \in \Phi} w_i \cdot n(\alpha_i,\widehat{\Upsilon}) - \log{\sum_{\Upsilon \in \Pi(\cC,\Phi_0)} e^{\sum_{\alpha_i \in \Phi} w_i \cdot n(\alpha_i,\Upsilon)}}$. This is the same as log-likelihood of $\widehat{\Upsilon}$ w.r.t.\ the MLN given by (\ref{eq:dual2_mln}), assuming that all the formulas in the MLN are proper. Thus, when the size of the domain of the training example $\widehat{\Upsilon}$ is the same as the cardinality of the domain of the modeled distribution, the max-entropy relational marginal problem in Model B is the same as maximum-likelihood estimation in MLNs.


\section{Realizability}

Not all relational marginals have a model for a given domain size (or even any model at all). This is a problem if we want to estimate relational marginals on some given global example $\widehat{\Upsilon}$ and then use them to obtain a distribution on global examples that have a different domain size \red{(which we address in Section {\em Estimation})}. The duality between maximum-likelihood estimation and max-entropy relational marginal problems discussed in the previous section only holds when the training data and the distribution that we want to model have the same domain size. 

In this section, we will show how to obtain relational marginals for any domain size. Accomplishing this requires replacing the consistency of relational marginal estimators with a weaker notion. However, we can still bound the difference between the unbiased and the adjusted estimates. More importantly, the difference will tend to zero as the domain size of the examples from which the respective marginals were estimated increases. 

First, it is easy to see that if a model exists for a given set of marginals and domain size, then there is also such a model for smaller domain sizes, as the next proposition asserts.\footnote{Proposition \ref{prop2} is one of the results which would not hold for Model B if we did not require injectivity of the randomly sampled substitutions in the definition of Model B.}

\begin{proposition}\label{prop2}
For Model A, if there is a width-$k$ model of $\Phi$ on domain of size $n$ then there is also a width-$k$ model of $\Phi$ on domain of size $m$ for any $m$ satisfying $n \geq m \geq k$. For Model B, if there is a model of $\Phi$ on domain of size $n$ then there is also a model of $\Phi$ on domain of size $m$ for any $m$ satisfying $n \geq m \geq \max_{\alpha \in \Phi} |\textit{vars}(\alpha)|$.
\end{proposition}

\noindent Next, we give an example of a relational marginal distributions that does not have a model for arbitrary domain cardinalities. 

\begin{example}
Let $\mathcal{L}$ consist of predicate symbols $e/2, r/1, g/1, b/1$ and $\Upsilon = (\cA,\cC)$ be a global example where
\begin{align*}
    \cA &= \{ e(v_1,v_2), e(v_2,v_1), e(v_2,v_3),e(v_3,v_2), e(v_1,v_3), \\
    &e(v_3,v_1), r(v_1), g(v_2), b(v_3) \},\quad\cC = \{ v_1, v_2, v_3 \}.
\end{align*}
Let $k = 2$ be the width of local examples. And let $F(X_1,X_2) \myeq X_1 \neq X_2 \wedge \neg e(X_1,X_1) \wedge e(X_1,X_2) \wedge e(X_2,X_1) \wedge \neg e(X_2,X_2)$. Then we can estimate, for instance, the following marginals from $\Upsilon$ under Model A:
\begin{align*}
P[ \exists X_1,X_2 : F(X_1,X_2) \wedge r(X_1) \wedge \neg g(X_1) \wedge \neg b(X_1) \wedge \\
      \wedge \neg r(X_2) \wedge g(X_2) \wedge \neg b(X_2)] = \frac{1}{3} \\
P[ \exists X_1,X_2 : F(X_1,X_2) \wedge r(X_1) \wedge \neg g(X_1) \wedge \neg b(X_1) \wedge \\
      \wedge \neg r(X_2) \wedge \neg g(X_2) \wedge b(X_2)] = \frac{1}{3} \\
P[ \exists X_1,X_2 : F(X_1,X_2) \wedge \neg r(X_1) \wedge g(X_1) \wedge \neg b(X_1) \wedge \\
      \wedge \neg r(X_2) \wedge \neg g(X_2) \wedge b(X_2)] = \frac{1}{3}
\end{align*}
(which can also be rewritten as probabilities of universally quantified formulas by negating the existentially quantified conjunctions). The global example $\Upsilon$ can be imagined as a complete directed graph (without self-loops) on $3$ vertices $v_1$, $v_2$, $v_3$ where each of the vertices is colored by one of the ``colors'' $r$, $g$, $b$. 

We claim that no distribution on global examples satisfies the above marginal probabilities for domain size greater than $3$.
This can be shown as follows. Using the intuitive view of $\Upsilon$ as a colored directed graph, the distribution on global examples of domain size e.g.\ $4$ would be a distribution on graphs with $4$ vertices. Such graphs would have to contain either two vertices not connected by an edge or two vertices connected by an edge but labeled with the same color or a vertex with no color. However, two such vertices would correspond to a local example which would otherwise have zero probability under the marginals estimated from $\Upsilon$, which are shown above. While the above reasoning is for Model A, a similar argument can be used to show the same issues  for Model B. 
\end{example}

One of the consequences of the above example is that the unbiased estimates of relational marginals from Proposition~\ref{prop:unbiased_estimate} cannot always be used for defining distributions of arbitrary domain sizes. \red{The section {\em Estimation} shows under which such unbiased estimates do exist, using the concept of {\em relational marginal polytopes}}.

In order to construct distributions for arbitrary domain sizes, which have relational marginals that are close to the relational marginals given by some global example $\Upsilon$, we will rely on the following construction which we call {\em expansion of global example}. 

\begin{definition}[Expansion of a global example]
Let $\Upsilon = (\cA,\cC)$ be a global example where $\cC = \{c_1,c_2,\dots,c_n\}$, and let $l$ be a positive integer. Then the {\em $l$-level expansion} of $\Upsilon$ is a global example $\Upsilon' = (\cA',\cC')$ given by: $\cA' = \{ a\theta : a \in \cA, \theta \in \Theta  \}$, $\cC' = \{ c_1, c_2, \dots, c_n, c_{n+1}, \dots, c_{l\cdot n}\}$. Here, constants $c_i$, $c_j$ are said to be congruent if $i \equiv j \mod n$.
Here, $c_{n+1}, \dots, c_{l \cdot n}$ are some arbitrary new constants and $\Theta$ is a set of all substitutions $\theta$ which satisfy the requirement that $c\theta$  is congruent with $c$ for each $c\in \mathcal{C}$.
\end{definition}

\noindent Next we illustrate the notion of expansion of examples.

\begin{example}
Let $\Upsilon = (\cA,\cC)$ be given by
\begin{align*}
    \cA = \{ e(c_1,c_2), e(c_2,c_3)\}, \cC = \{ c_1,c_2,c_3 \}.
\end{align*}
\blue{Interpreting the predicate $e$ as edge, the corresponding graph looks like:}

\begin{center}
\resizebox{0.3\textwidth}{!}{
\tikzset{
main node/.style={circle,fill=white!11,draw,minimum size=0.1cm,inner sep=0pt},
other node/.style={rectangle,fill=white!11,minimum size=0.1cm,inner sep=0pt},
}
\tikzset{edge/.style = {->,> = latex'}}
\begin{tikzpicture}

\node[main node] (1) {$c_1$};
\node[main node] (2) [right = 2cm of 1] {$c_2$};
\node[main node] (3) [right = 2cm of 2] {$c_3$};

\draw[edge] (1) to (2);
\draw[edge] (2) to (3);

\end{tikzpicture}}
\end{center}

\noindent Then the $2$-level expansion of $\Upsilon$ is $\Upsilon' = (\cA',\cC')$ where
\begin{align*}
    \cA &= \{ e(c_1,c_2), e(c_2,c_3), e(c_4,c_5), e(c_5,c_6), e(c_1,c_5),\\
    &e(c_2,c_6), e(c_4,c_2), e(c_5,c_3) \},\cC = \{ c_1,c_2,c_3, c_4,c_5,c_6 \}.
\end{align*}
\blue{Interpreting the predicate $e$ again as edge, the expansion corresponds to the following graph:}

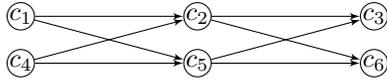
\begin{center}
\resizebox{0.3\textwidth}{!}{
\tikzset{
main node/.style={circle,fill=white!11,draw,minimum size=0.1cm,inner sep=0pt},
other node/.style={rectangle,fill=white!11,minimum size=0.1cm,inner sep=0pt},
}
\tikzset{edge/.style = {->,> = latex'}}
\begin{tikzpicture}

\node[main node] (1) {$c_1$};
\node[main node] (2) [right = 2cm of 1] {$c_2$};
\node[main node] (3) [right = 2cm of 2] {$c_3$};

\node[main node] (4) [below = 0.25cm of 1] {$c_4$};
\node[main node] (5) [below = 0.25cm of 2] {$c_5$};
\node[main node] (6) [below = 0.25cm of 3] {$c_6$};

\draw[edge] (1) to (2);
\draw[edge] (2) to (3);
\draw[edge] (1) to (5);
\draw[edge] (2) to (6);
\draw[edge] (4) to (2);
\draw[edge] (5) to (3);
\draw[edge] (4) to (5);
\draw[edge] (5) to (6);

\end{tikzpicture}}
\end{center}

\noindent The width-$2$ marginal probabilities on $\Upsilon$ and $\Upsilon'$ are:
{\small \begin{align*}
    P_{\Upsilon,2}(\omega_1) = \frac{1}{3} && P_{\Upsilon',2}(\omega_1) = \frac{7}{15}  && \omega_1 = (\{ \}, \{1,2 \}) \\
    P_{\Upsilon,2}(\omega_2) = \frac{1}{3} && P_{\Upsilon',2}(\omega_2) = \frac{4}{15}  && \omega_2 = (\{ e(1,2) \}, \{1,2 \}) \\
    P_{\Upsilon,2}(\omega_3) = \frac{1}{3} && P_{\Upsilon',2}(\omega_3) = \frac{4}{15}  && \omega_3 = (\{ e(2,1) \}, \{1,2 \}) \\
    P_{\Upsilon,2}(\omega) = 0 && P_{\Upsilon',2}(\omega) = 0 && \omega \not\in \{\omega_1,\omega_2,\omega_3 \}
\end{align*}}
The differences between the marginal probabilities given by $\Upsilon$ and $\Upsilon'$ are at most $\frac{2}{15}$ in this case, which is quite high. However, it follows from what we show in turn that this is mostly because of the small size of $\Upsilon$. For larger global examples, the difference between the marginals obtained from them and from their expansions will tend to be smaller.
\end{example}

Importantly, it is possible to bound the difference of the parameters obtained on expansions of global examples and the unbiased estimates obtained on the original examples.

\begin{proposition}\label{prop:difference}
Let $\Upsilon = (\cA,\cC)$ be a global example and $\Upsilon'$ its $l$-level expansion and let $n = |\mathcal{A}|$ and $k$ be the width of local examples. Then for any formula $\alpha$:
\begin{multline*}
|P_{\Upsilon,k}(\alpha)-P_{\Upsilon',k}(\alpha)| \leq 1-\left(\frac{n-k+1}{n}\right)^{k-1}
\end{multline*}
for Model A, and
\begin{multline*}
|Q_{\Upsilon}(\alpha)-Q_{\Upsilon'}(\alpha)| \leq 1-\left(\frac{n-k_\alpha+1}{n}\right)^{k_\alpha-1} 
\end{multline*}
for Model B, where $k_\alpha = |\textit{vars}(\alpha)|$.
\end{proposition}

\noindent Note that the difference between the true and modeled probabilities of a fixed formula $\alpha$ decays as $O\left(\frac{1}{n}\right)$.


The techniques described in this section have the following limitation. If we have a set of hard rules $\Phi_0$ which are satisfied by a given $\Upsilon$, these rules may not be satisfied in an expansion of $\Upsilon$. This is not just a limitation of our method though. There are cases where it is not possible to extend a given $\Upsilon$ while satisfying the constraints (this is because we allow the use of equality in the formulas and because we use the unique names assumption). However, if the example $\Upsilon$ is large enough and it satisfies the hard rules, then the number of violations of these rules will be small, which follows again from Proposition \ref{prop:difference}. 

In fact, it seems to be a desirable property that formulas $\alpha$ satisfying $P_{\Upsilon,k}(\alpha) = 1$ do not have to be treated as completely hard rules but as rules that ``mostly'' hold if they are learned from $\Upsilon$, since it may be that they are not really rules that should always hold. Yet, if we actually took them as hard rules we would be forced to assign probability $0$ to any example that violates them. It is possible to use the idea of expansions to obtain a distribution in which any formula $\alpha$ has nonzero probability and all properties of expansions are still preserved (such as those from Proposition \ref{prop:difference}). This can be achieved by randomly sampling additional atoms containing only congruent constants and adding them to the respective expansion. If we use a sufficiently high level of the expansion (at least $k$ for Model A and at least $\max_{\alpha \in \Phi}|\textit{vars}(\alpha)|$ for Model B) then the probability of any formula will be nonzero and not equal to one w.r.t.\ the distribution induced by the expansions with the sampled atoms.



\section{Relational Marginal Polytopes}

In this section, we define another important concept called {\em relational marginal polytope}, which will be used in the next section where we deal with estimation errors. 

\begin{definition}[Relational marginal polytope for Model A]
Let $k,m \in \mathds{N}$ and $\Phi = \{ \alpha_1,\dots,\alpha_l \}$ be a set of formulas and $\Phi_0$ be a set of hard rules. Let $\cC = \{1,\dots,m\}$ and $\cC_k$ be the set of size-$k$ subsets of $\cC$. Then, for Model A, the relational marginal polytope of $\Phi$ of width $k$ and cardinality $m$ w.r.t.\ the hard rules from $\Phi_0$ is the convex hull of the set
$$\left\{ \left. \left( \frac{\#_k(\alpha_1,\Upsilon)}{|\cC_k|}, \dots, \frac{\#_k(\alpha_l,\Upsilon)}{|\cC_k|}  \right) \right| \Upsilon \in \Pi(\cC,\Phi_0) \right\}.$$
Let $\Theta_{\alpha_i}$ be the set of all injective substitutions from variables of $\alpha_i$ to constants from $\cC$. Then, for Model B, the relational marginal polytope of $\Phi$ of cardinality $m$ w.r.t.\ the hard rules from $\Phi_0$ is the convex hull of the set
$$\left\{ \left. \left( \frac{n(\alpha_1,\Upsilon)}{|\Theta_{\alpha_1}|}, \dots, \frac{n(\alpha_l,\Upsilon)}{|\Theta_{\alpha_l}|}  \right) \right| \Upsilon \in \Pi(\cC,\Phi_0) \right\}.$$
\end{definition}




\noindent Any realizable set of relational marginals for Model A and Model B naturally corresponds to a point in the respective polytope. In the remainder of this paper, we only consider the cases when the relational marginal polytope is full-dimensional, that is, it does not live in a lower dimensional subspace which could happen if some of the relational marginals that define it were linearly dependent.

We will also need the concept of $\eta$-interior of a relational marginal. We say that a point $y$ is in the $\eta$-interior of a relational marginal polytope $P$ if $P$ contains a ball of radius $\eta$ with center in $y$. Using Proposition \ref{prop2} and Proposition \ref{prop:difference}, we can show the following both for Model A and Model B. 
\begin{proposition}\label{prop:interiority}
Let $\theta$ be a vector representing the values of a set of relational marginals given by formulas from a set $\Phi = \{\alpha_1,\dots,\alpha_l \}$. Let $k$ be the width of the relational marginals of Model A or $k = \max_{\alpha_i} |\textit{vars}(\alpha_i)|$ for Model B. Let the set of hard rules $\Phi_0$ be empty. If $\theta$ is in the $\left(\eta+\sqrt{l}\left(1-\left(\frac{m-k+1}{m}\right)^{k-1}\right)\right)$-interior of the relational marginal polytope of $\Phi$ of domain-size $m$ then it is also in the $\eta$-interior of the relational marginal polytope of $\Phi$ for any domain size $m'$.
\end{proposition}








\section{Estimation}

In this section, we present error bounds for the estimation of relational marginals. We start by defining the learning setting. Clearly, we need some assumptions on the training and test data and their relationship (otherwise one could always come up with a setting in which the error can be arbitrarily large). In order to stay close to realistic settings we assume that there is some large global example $\aleph = (\cA_\aleph,\cC_\aleph)$ that is not available and that represents the ground truth. That is what we essentially want to estimate for a given formula $\alpha$ is $P_{\aleph,k}(\alpha)$, but we do not have access to whole $\aleph$. Imagine for instance that $\aleph$ is the human gene regulatory network or Facebook. We assume that there is a process that samples size-$m$ subsets of $\cC_\aleph$ uniformly and that we have access to one such sample $\cC_{\Upsilon}$ and also to the respective induced $\Upsilon = \aleph \langle \cC_\Upsilon \rangle$. So, for a given formula $\alpha$, we need to estimate $P_{\aleph,k}(\alpha)$ using the available example $\Upsilon$ and the estimate needs to be realizable (otherwise the optimization problem would have no solution and the duality would also not hold). This is a reasonably realistic setting\footnote{What might differ in realistic settings is the sampling process. We briefly discuss this in section {\em Conclusions}.} as in practice we also do not have a distribution over different Facebooks but there is one Facebook and we want to model it based on a sample that is available to us.

We now provide theoretical upper bounds for the expected error of the estimates of $P_{\aleph,k}(\alpha)$ assuming the just described learning setting. However, we first need to describe the estimators.
Based on the results from the previous sections, the estimator works as follows. Given a global example $\Upsilon = (\cA_\Upsilon,\cC_\Upsilon)$ and an integer $n$, which is the size of the domain of the modelled distribution (e.g.\ $n$ can be size of $\aleph$'s domain if it is known), we construct the $l$-level expansion $\Upsilon^{(l)}$ of $\Upsilon$, where $l = \lceil n/|\cC_\Upsilon| \rceil$, and we use it to estimate the parameters as $\widehat{P}_\alpha = P_{\Upsilon^{(l)},k}(\alpha)$ for Model A and as $\widehat{Q}_\alpha = Q_{\Upsilon^{(l)}}(\alpha)$ for Model B.
The following proposition introduces an upper bound for the expected error of the estimated parameters.

\begin{proposition}\label{prop:main-estimation}
Let $m$ and $n$ be positive integers, $\alpha$  a closed formula and let $k$ be the width of local examples. Let $\aleph = (\cA_\aleph,\cC_\aleph)$ be a global example, $\cC_\Upsilon$ be sampled uniformly among all size-$m$ subsets of $\cC_\aleph$ and $\Upsilon = \aleph\langle \cC_\Upsilon \rangle$. Let $\widehat{A}_{\aleph} = P_{\aleph,k}(\alpha)$. Let $\widehat{B}_\Upsilon$ be an estimate computed from the $l$-level expansion of $\cC_\Upsilon$. Then 
$$\expect{\left| \widehat{A}_\aleph - \widehat{B}_{\Upsilon} \right|} \leq 1-\left(\frac{m-k+1}{m}\right)^{k-1} + \sqrt{\frac{1+2\log 2}{4\lfloor m/k \rfloor}}.$$
In the case of model B, the same upper bound holds if we choose $k = |\textit{vars}(\alpha)|$ and $\widehat{A}_{\aleph} = Q_{\aleph}(\alpha)$.
\end{proposition}
\noindent The proof of this proposition, which is contained in the online appendix, is based on a deviation bound for a randomized estimator of $\widehat{B}_{\Upsilon}$. 

It is possible to improve the estimation in some cases. If the vector corresponding to the marginals $\Phi$ estimated from $\Upsilon$ is guaranteed to be in the $\left(\eta+\sqrt{l}\left(1-\left(\frac{m-k+1}{m}\right)^{k-1}\right)\right)$-interior of the relational marginal polytope of domain of size $m = |\cC_\Upsilon|$ for some $\eta > 0$, where $k$ is the width of local examples for Model A (or $\max_{\alpha \in \Phi}|\textit{vars}(\alpha)|$ for Model B), then, by Proposition \ref{prop:interiority}, we can estimate the parameters directly from $\Upsilon$ without constructing its expansion. We then have the following improved bound:
$$
\expect{\left| \widehat{A}_\aleph - \widehat{B}_{\Upsilon} \right|} \leq \sqrt{(1+2\log 2)/(4\lfloor m/k \rfloor)}.
$$

It is interesting to note that the lower-bound on effective sample size obtained from these bounds is $\lfloor m/k \rfloor$, which is also the maximum number of non-overlapping size-$k$ subsets of $\cC_\Upsilon$. A consequence for learning the parameters of models such as MLNs (which corresponds to relational marginal problems in Model B) is that this bound is inversely proportional to the number of variables in the used formulas, which also suggests an explanation for why learning with longer formulas is difficult. 




\section{Related Work}


The relational marginals from Model A were recently introduced \cite{kuzelka.ijcai.17}. However, they were only studied in a possibilistic setting, which differs substantially from the probabilistic maximum-entropy setting that we considered. The idea of using random substitutions (Model B) goes back to \cite{bacchus_halpern_koller} who, however, only considered unary predicates. Schulte et al.\ (\citeyear{schulte}) used the random substitutions semantics to define a relational Bayesian network model for population statistics. However, their model is, not based on any underlying ground model, and it is unclear whether the distributions are always realizable by a ground model. 

In the more restricted setting of exponential random graph models (ERGMs, Chatterjee and Diakonis \citeyear{chatterjee2013estimating}), a formally similar duality, based on densities of graph homomorphisms, has previously been established. To the best of our knowledge, however, such a duality has never been established in an SRL setting. In fact, even for ERGMs this duality has not yet been exploited for estimating parameters for models of different domain sizes, which is one of the key contributions of our work.

Certain statistical properties of learning have been already studied for SRL models. Xiang and Neville (\citeyear{xiang.neville}) studied consistency but postulated rather strong assumptions\footnote{The assumptions used in their work were {\em weak dependency} and bounded degree of graph nodes.}, as a result of which their results are not comparable with ours. Their approach also differs in that it only considers 
distributions of labels conditioned on the underlying graph structure. It is interesting that a statistical estimation problem equivalent to the estimation of parameters in Model A has also been studied in the literature. In Nandi and Sen~(\citeyear{nandi1963properties}), the variance of an estimator, equivalent to the unbiased estimator for Model A, was given. However, we are not aware of any work showing a deviation bound in the same setting, which was needed to establish the bound on expected error in Proposition \ref{prop:main-estimation}. Interestingly, the effective sample size $m/k$ stemming from the work \cite{nandi1963properties} for variance of the estimator is almost the same as the effective sample size $\lfloor m/k \rfloor$ stemming from our deviation and expected-error bounds. This actually suggests that our bounds are rather tight. 
There were many works on U-statistics \cite{hoeffding1948class} which are related as well but they rely on assumptions that are generally not realistic for SRL and, in particular, are not applicable to our setting; the work of Nandi and Sen~(\citeyear{nandi1963properties}) is an exception.

Schulte (\citeyear{schulte2012challenge})\footnote{We are grateful to Oliver Schulte for pointing us to this paper.} formulated several open questions for SRL that directly relate to the topics that we cover in this paper. One of the open questions was whether Markov logic networks satisfy the principle of equivalence between instance-level and population-level or class-level marginal probabilities. Here, population-level probabilities refer to frequencies estimated from the relational data (e.g.\ frequency in the population of pairs of people who are friends) and instance-level probabilities refer to probabilities regarding individuals (e.g.\ what is the probability that Alice and Bob are friends assuming we do not know anything else about them). The duality between relational marginal problems and maximum likelihood estimation in Model B, which corresponds to Markov logic networks, that we established for the case of equal domain sizes of the training and test data, provides to some extent a positive answer to this open question.

\section{Conclusions}

In this paper, we have introduced and studied relational marginal problems. Interestingly, this perspective enables learning a model that is applicable to data sets whose domain sizes differ from that of the training data. We established a relational counterpart of the classical duality between maximum-likelihood and max-entropy marginal problems. Then, we showed how to estimate and adjust parameters of the marginals in order to guarantee their realizability. We complemented these results by providing bounds on the expected errors of the estimates in a reasonable setting. 

We believe that due to the simplicity and transparency of the learning setting that we introduced, this setting could play a similar role for SRL as the standard i.i.d. statistical learning setting plays for learning in propositional domains \cite{Vapnik:1995:NSL:211359}. That is, as an idealized setting that is suitable for theoretical study, but that is not too far from settings that one encounters in reality. Still, it would be possible to extend the learning setting to make it more realistic. In particular, the sampling process that creates the training examples could be replaced by another sampling process that would take into account the structure of the relational data. That would probably make estimation of parameters and derivation of error bounds significantly more complex, and hence arguably less illuminating, which is why we leave it for future work.

\subsection{Acknowledgment}
\blue{The authors would like to thank the anonymous reviewers for their helpful comments. This work was supported by a Leverhulme Trust grant (RPG-2014-164) and ERC Starting Grant 637277. JD is partially supported by the KU Leuven Research Fund  (C14/17/070,C22/15/015,C32/17/036), and FWO-Vlaanderen (G.0356.12, SBO-150033). YW is partially supported by Guangdong Shannon Intelligent Tech.\ co., Ltd.}



\fi

\ifappendix


\appendix

\section{Appendix: Duality}

In this section we prove the duality results from the main text. First we recall the setting and notations. Let $\cC$ be a set of constants, $\cA$ the set of all atoms over $\cC$ based on some given first-order language and let $\mathcal{C}_k$ denote the set of all $k$-element subsets of $\cC$. Next, let $\Phi$ be a set of relational marginals and $\Phi_0$ be a set of hard constraints, i.e.\ formulas $\alpha$ such that if $\Upsilon \not\models \alpha$ then $P(\Upsilon) = 0$. Let us also assume that there exists at least one distribution $P'$ which is a model of $\Phi$ and which satisfies the positivity condition:  $P'(\Upsilon) > 0$ for all $\Upsilon$ satisfying the hard constraints.

The optimization problem is then given by:

\begin{equation}
    \sup_{\{ P_\Upsilon : \Upsilon \in \Pi(\cC,\Phi_0) \}}  \sum_{\Upsilon \in \Pi_{n}(\cC,\Phi_0)} P_{\Upsilon} \log{\frac{1}{P_\Upsilon}} \quad \textit{ s.t.}\label{eq:maximum_entropy_criterion2}
\end{equation}
\begin{multline}\label{eq:maxent_marginal_constraints2}
    \forall (\alpha_i,\theta_i) \in \Phi : \\
    \frac{1}{|\cC_k|}\sum_{\mathcal{S} \in \cC_k } \sum_{\Upsilon \in \Pi(\cC,\Phi_0)} \mathds{1}(\Upsilon\langle \mathcal{S} \rangle \models \alpha_i) \cdot P_\Upsilon = \theta_i
\end{multline}
\begin{align}
\forall \Upsilon \in \Pi_{n}(\cC,\Phi_0) : P_{\Upsilon} \geq 0, &&\sum_{\Upsilon \in \Pi(\cC,\Phi_0)} P_{\Upsilon} = 1\label{eq:maximum_entropy_cnormalization2}
\end{align}

Changing maximization of entropy to minimization of negative entropy, we obtain a convex optimization problem. Next, we construct the Lagrangian $L(P,\lambda,z)$ (here, $P$ denotes the vector of all $P_\Upsilon \in \Pi(\cC,\Phi_0)$, $\lambda$ the vector of all $\lambda_i$'s, and we will ignore the non-negativity constraints for the moment as it will turn out that they are enforced implicitly):

\begin{multline}
L(P,\lambda,z) = \sum_{\Upsilon \in \Pi(\cC,\Phi_0)} P_\Upsilon \cdot \log{P_\Upsilon} - \\
-\frac{1}{|\mathcal{C}_k|} \sum_{\alpha_i \in \Phi}  \sum_{\mathcal{S} \in \mathcal{C}_k} \sum_{\Upsilon \in \Pi(\cC,\Phi_0)} \lambda_i \cdot \mathds{1} (\Upsilon \langle \mathcal{S} \rangle \models \alpha_i) \cdot P_{\Upsilon} + \\
+ \sum_{\alpha_i \in \Phi} \lambda_i \cdot \theta_i + z - z \sum_{\Upsilon \in \Pi(\cC,\Phi_0)} P_\Upsilon\label{eq:lagrangian}
\end{multline}

\noindent 
To find the stationary points of $L(P,\lambda,z)$ w.r.t.\ $P$, we take the partial derivatives of the Lagrangian w.r.t.\ $P_{\Upsilon_j}$ and set them equal to zero:
{\small \begin{multline}
\log P_{\Upsilon_j} - \frac{1}{|\cC_k|} \sum_{\alpha_i \in \Phi} \sum_{\mathcal{S} \in \mathcal{C}_k} \lambda_i \cdot \mathds{1} (\Upsilon_j \langle \mathcal{S} \rangle \models \alpha_i) -z-1 = 0\label{eq:partial2}
\end{multline}}
\noindent Next we define $\#_k(\alpha,\Upsilon) = |\{ \mathcal{S} \in \mathcal{C}_k : \Upsilon\langle \mathcal{S} \rangle \models \alpha \}|$. That is $\#_k(\alpha,\Upsilon)$ is a function counting the number of sets $\mathcal{S} \in \cC_k$ such that the formula $\alpha$ is true in the restriction of $\Upsilon$ to constants in $\mathcal{S}$ (i.e.\ those for which $\Upsilon \langle \mathcal{S} \rangle \models \alpha$). This allows us to rewrite (\ref{eq:partial2}) as follows:

\begin{multline}
\log P_{\Upsilon_j} - \sum_{\alpha_i \in \Phi} \lambda_i \frac{\#_k(\alpha_i,\Upsilon_j)}{|\cC_k|} - z-1 = 0\label{eq:partial3}
\end{multline}

\noindent From this, we get:
\begin{equation}
    P_{\Upsilon_j} = \exp{\left(z+1 + \sum_{\alpha_i \in \Phi} \lambda_i \frac{\#_k(\alpha_i,\Upsilon_j)}{|\cC_k|}\right)}\label{eq:distr1}
\end{equation}
Since we have $\sum_{\Upsilon \in \Pi(\cC,\Phi_0)} P_\Upsilon = 1$, $\exp{(z+1)}$ is a normalization constant.

By further algebraic manipulations, combining (\ref{eq:lagrangian}) and (\ref{eq:distr1}) we obtain the Lagrangian dual problem (where $\lambda_i \in \mathbb{R}$) is to maximize:
\begin{multline}
\sum_{\alpha_i \in \Phi} \lambda_i \theta_i - \log{\sum_{\Upsilon \in \Pi(\cC,\Phi_0)} e^{\sum_{\alpha_i \in \Phi} \lambda_i \frac{\#_k(\alpha_i,\Upsilon_j)}{|\cC_k|}}}\label{eq:dual}
\end{multline}
Due to the positivity assumption, Slater's condition \cite{boyd2004convex} is satisfied and strong duality holds. 

\begin{observation}
The same reasoning can be also applied when we have multiple sets of relational marginals $\Phi_1$, $\Phi_2$, $\dots$, $\Phi_n$ ($\Phi = \Phi_1 \cup \dots \cup \Phi_n$) containing relational marginals of widths $1$, $2$, $\dots$, $n$, respectively. The only difference is that the denominators $|\cC_k|$ will vary according to widths ($k_i$) of the respective marginals ($\alpha_i$) and we will denote the respective sets of subsets of $\cC$ by $\cC_{k_i}$.
\end{observation}

Let us perform a change of variables $w_i := \lambda_i/|\cC_{k_i}|$ and denote the normalization constant of the distribution as $Z$. This gives us 
\begin{equation}
    P(\Upsilon) = \frac{1}{Z}\exp{\left(\sum_{\alpha_i \in \Phi} w_i \cdot \#_k(\alpha_i,\Upsilon)\right)}
\end{equation}
for the distribution and
\begin{multline}
\sum_{\alpha_i \in \Phi} w_i \theta_i |\cC_{k_i}| - \log{\sum_{\Upsilon \in \Pi(\cC,\Phi_0)} e^{\sum_{\alpha_i \in \Phi} w_i \cdot \#_k(\alpha_i,\Upsilon)}}\label{eq_app:dual2}
\end{multline}
for the optimization criterion of the dual problem. Assuming that the marginals were estimated from a global example $\widehat{\Upsilon} \in \Pi(\cC,\Phi_0)$ (note that here the domain $\cC$ is the same as the domain of the global examples over which the distribution is defined) and that they still satisfy the positivity assumption, we can also rewrite (\ref{eq_app:dual2}) as
{\small \begin{multline}
\sum_{\alpha_i \in \Phi} w_i \cdot \#_k(\alpha_i,\widehat{\Upsilon}) - \log{\sum_{\Upsilon \in \Pi(\cC,\Phi_0)} e^{\sum_{\alpha_i \in \Phi} w_i \cdot \#_k(\alpha_i,\Upsilon)}}\label{eq:dual3}
\end{multline}}

The reasoning is completely analogical for Model B. The only difference is that $\#_k(\alpha,\Upsilon)$ is replaced by $n(\alpha,\Upsilon)$ which counts the number of true groundings of $\alpha$ in $\Upsilon$.

\section{Appendix: Proofs}

In this section we give proofs of the propositions from the main text.

\noindent {\bf Proposition \ref{prop:unbiased_estimate}.}
{\em Let $P(\Upsilon)$ be a distribution on domain size $n$ and $k \leq n$ be an integer. Let 
$\Omega_{\alpha} = \{ \Upsilon \langle \mathcal{S} \rangle : \mathcal{S} \subseteq \cC, |\mathcal{S}| = k, \Upsilon \langle \mathcal{S} \rangle \models  \alpha \}$ 
where $\Upsilon = (\cA, \cC)$ is sampled according to the distribution $P(\Upsilon)$. Then, for a closed formula $\alpha$,
$$\widehat{p}_{\alpha} = |\Omega_\alpha| \cdot \left( \begin{array}{c} n \\ k \end{array} \right)^{-1} $$
is an unbiased estimate of $P_k(\alpha)$.}
\begin{proof}
We have:
\begin{multline*}
\mathbf{E}[\widehat{p}_{\alpha}] = \sum_{\Upsilon = (\cA,\cC) \in \Omega} P(\Upsilon) \cdot \left( \begin{array}{c} n \\ k \end{array} \right)^{-1} \cdot \\
\cdot |\{ \mathcal{S} : \mathcal{S} \subseteq \cC, |\mathcal{S}| = k, \Upsilon \langle \mathcal{S} \rangle \}| = \sum_{\Upsilon \in \Omega} P(\Upsilon) \cdot P_{\Upsilon,k}(\alpha) = \\ 
= \sum_{\omega : \omega \models \alpha} \sum_{\Upsilon \in \Omega} P(\Upsilon) \cdot P_{\Upsilon,k}(\omega) = P(\alpha).
\end{multline*}
\end{proof}

\noindent {\bf Proposition \ref{prop2}.}
{\em For Model A, if there is a width-$k$ model of $\Phi$ on domain of size $n$ then there is also a width-$k$ model of $\Phi$ on domain of size $m$ for any $m$ satisfying $n \geq m \geq k$. For Model B, if there is a model of $\Phi$ on domain of size $n$ then there is also a model of $\Phi$ on domain of size $m$ for any $m$ satisfying $n \geq m \geq \max_{\alpha \in \Phi} |\textit{vars}(\alpha)|$.}
\begin{proof}
Let $P(.)$ be a model of $\Phi$ on domain of size $n$. Let $P'(.)$ be a distribution given by the following process: (i) sample $\Upsilon = (\cA, \cC)$ according to $P(.)$, (ii) sample a subset $\cC'$ from $\cC$ uniformly among all subsets of $\cC$ of cardinality $m$, and (iii) set $\Gamma := \Upsilon \langle \cC' \rangle$. We claim that $P'(.)$ is also a width-$k$ model of $\mathcal{M}$. To show that this is true it is enough to demonstrate that $P_k(\omega) = P'_k(\omega)$ for all width-$k$ local examples $\omega$, where $P_k(.)$ and $P'_k(.)$ are relational marginal distributions induced by $P(.)$ and $P'(.)$, respectively. This is straightforward to show. It suffices to notice that $P_k(\omega)$ does not change when we modify the relational marginal sampling process so that it first samples $\Upsilon = (\cA, \cC)$, then it samples uniformly a subset $\cC'$ of $\cC$ of cardinality $m$, then it samples uniformly a subset $\mathcal{S}$ of $\cC'$ and then a local example from $\Upsilon[\mathcal{S}]$. But this process is equivalent to sampling local examples from the global examples sampled according to the distribution $P'(.)$. Hence, $P_k(\omega) = P'_k(\omega)$. 

The argument is completely analogical for Model B\footnote{Here, the condition $n \geq m \geq \max_{\alpha \in \Phi} |\textit{vars}(\alpha)|$ is here so that we could actually sample injective substitutions.}.
\end{proof}

It is worth noting here why the Proposition \ref{prop2} would not hold for Model B if we did not require the sampled substitutions of Model B to be injective. In such a case, the distribution of local examples $\omega$ would not be the same if we first sampled a subset $\cC'$ of size $m$ from $\cC$. For instance, for a formula $\alpha$ with two variables $A$ and $B$, the proportion of samples mapping $A$ and $B$ to the same constant would be $\frac{1}{n}$ for sampling directly from $\cC$ but $\frac{1}{m}$ when first sampling $\cC'$ from $\cC$ and then sampling the substitutions for $\cC'$. So the two distributions would be no longer equivalent.

\begin{lemma}\label{prop:propgamma}
Let $\Upsilon = (\cA,\cC)$ be a global example and $\Upsilon' = (\cA',\cC')$ its $l$-level expansion. Let $n = |\cC|$. Then there exists a distribution $\lambda(.)$ such that
$P_{\Upsilon',k}(\omega) = (1-\gamma) P_{\Upsilon,k}(\omega) + \gamma \lambda(\omega)$.
Here $\gamma$ is the probability that a randomly sampled subset of $k$ constants contains at least two congruent constants, i.e.
$$\gamma =  1-\frac{\left( \begin{array}{c} n \\ k \end{array} \right) \cdot l^k}{\left( \begin{array}{c} n \cdot l \\ k \end{array} \right)}$$
\end{lemma}
\begin{proof}
We prove this proposition for Model A; the arguments for Model B are essentially the same and therefore we omit them.

For notational convenience, let us define an indicator function $\mathbbm{c}(\mathcal{S})$ which equals $1$ when $\mathcal{S}$ contains two congruent constants and is $0$ otherwise. We decompose the probability $P_{\Upsilon',k}(\omega)$ as:
\begin{multline*}
P_{\Upsilon',k}(\omega) = \\
\frac{1}{|\Upsilon'[\mathcal{S}]|} P_{\mathcal{S}}[\omega \in \Upsilon'[\mathcal{S}] | \mathbbm{c}(\mathcal{S}) = 1] \cdot P_\mathcal{S}[\mathbbm{c}(\mathcal{S}) = 1] + \\
+ \frac{1}{|\Upsilon'[\mathcal{S}]|} P_{\mathcal{S}}[\omega \in \Upsilon'[\mathcal{S}] | \mathbbm{c}(\mathcal{S}) = 0] \cdot P_\mathcal{S}[\mathbbm{c}(\mathcal{S}) = 0]
\end{multline*}
where $P_{\mathcal{S}}[.]$ denotes probability under sampling of $\mathcal{S}$ according to Model A. Using elementary combinatorial reasoning, we have $P_\mathcal{S}[\mathbbm{c}(\mathcal{S}) = 1] = \gamma$, $P_\mathcal{S}[\mathbbm{c}(\mathcal{S}) = 0] = 1-\gamma$ and we can set 
$\lambda(\omega)  = \frac{1}{|\Upsilon'[\mathcal{S}]|} P_{\mathcal{S}}[\omega \in \Upsilon'[\mathcal{S}] | \mathbbm{c}(\mathcal{S}) = 1]$.
It remains to analyze $\frac{1}{|\Upsilon'[\mathcal{S}]|} P_{\mathcal{S}}[\omega \in \Upsilon'[\mathcal{S}] | \mathbbm{c}(\mathcal{S}) = 0]$. First, we observe that this is the probability of the local example $\omega$ under the distribution ``induced'' by a uniform distribution over size-$k$ subsets of $\cC'$ which do not contain congruent constants. The same distribution can also be obtained as follows. First, pick uniformly at random a size-$k$ subset $\mathcal{S}'$ of $\cC$ and then assign, again uniformly and independently at random, to each of the constants in $\mathcal{S}'$ a constant sampled from the constants congruent to it in $\cC'$. Then we may notice that it follows from the construction of expansions that the isomorphism class of the induced local example $\omega$ to be sampled is already determined when we sample the set $\mathcal{S}'$ and it is equal to $\Upsilon[\mathcal{S}']$. It follows that $\frac{1}{|\Upsilon'[\mathcal{S}]|} P_{\mathcal{S}}[\omega \in \Upsilon'[\mathcal{S}] | \mathbbm{c}(\mathcal{S}) = 0] = P_{\Upsilon,k}(\omega)$ which finishes the proof.
\end{proof}

\begin{lemma}\label{lemma:ineq}
Let $P$ and $\lambda$ be distributions over $\Omega$ and $\alpha$ be a closed and constant-free formula. Then
$$\left|\sum_{\omega \in \Omega : \omega \models \alpha} P(\omega) - \lambda(\omega) \right| \leq 1$$
\end{lemma}
\begin{proof}
If $\sum_{\omega \in \Omega : \omega \models \alpha} P(\omega) - \lambda(\omega) \geq 0$ then we have
\begin{equation*}
    \sum_{\omega \in \Omega : \omega \models \alpha} P(\omega) - \lambda(\omega) \leq 1-\sum_{\omega \in \Omega : \omega \models \alpha} \lambda(\omega) \leq 1
\end{equation*}
where the first and second inequality follow from the fact that $P(\omega)$ and $\lambda(\omega)$ most sum up to $1$ over $\Omega$ and that they are positive. We can reason completely analogically for $\sum_{\omega \in \Omega : \omega \models \alpha} P(\omega) - \lambda(\omega) \leq 0$. The correctness of the lemma follows easily.
\end{proof}

\noindent {\bf Proposition \ref{prop:difference}.}
{\em Let $\Upsilon = (\cA,\cC)$ be a global example and $\Upsilon'$ its $l$-level expansion and let $n = |\mathcal{A}|$ and $k$ be the width of local examples. Then for any formula $\alpha$:
\begin{multline*}
|P_{\Upsilon,k}(\alpha)-P_{\Upsilon',k}(\alpha)| \leq 1-\left(\frac{n-k+1}{n}\right)^{k-1}
\end{multline*}
for Model A, and
\begin{multline*}
|Q_{\Upsilon}(\alpha)-Q_{\Upsilon'}(\alpha)| \leq 1-\left(\frac{n-|\textit{vars}(\alpha)|+1}{n}\right)^{|\textit{vars}(\alpha)|-1} 
\end{multline*}
for Model B.}
\begin{proof} We first prove the proposition for Model A. Let $\gamma$ be as in Lemma \ref{prop:propgamma}. We have $\left| P_{\Upsilon,k}(\alpha)-P_{\Upsilon',k}(\alpha) \right| =$
\begin{align*}
    \left| \sum_{\omega : \omega \models \alpha} \left( P_{\Upsilon,k}(\omega) - (1-\gamma) \cdot P_{\Upsilon,k}(\omega) - \gamma\cdot \lambda(\omega) \right) \right|
\end{align*}
\noindent where $\lambda(.)$ is some suitable probability distribution. This can be rewritten as
\begin{multline*}
    |P_{\Upsilon,k}(\alpha)-P_{\Upsilon',k}(\alpha)| = \gamma \cdot \left|  \sum_{\omega : \omega \models \alpha} \left(P_{\Upsilon,k}(\omega)-\lambda(\omega) \right)\right| \leq \\
    \leq \gamma
    = 1-\frac{l^k \cdot n \cdot (n-1) \cdot \dots \cdot (n-k+1) }{n \cdot l \cdot (n\cdot l-1) \cdot \dots \cdot (n\cdot l-k+1)}
\end{multline*}
where the first inequality follows from Lemma \ref{lemma:ineq}. Next, we can verify that 
\begin{align*}
1-\frac{l^k \cdot n \cdot (n-1) \cdot \dots \cdot (n-k+1) }{n \cdot l \cdot (n\cdot l-1) \cdot \dots \cdot (n\cdot l-k+1)} = \\
= 1-\frac{ n \cdot (n-1) \cdot \dots \cdot (n-k+1) }{n \cdot (n-\frac{1}{l}) \cdot \dots \cdot (n-\frac{k+1}{l})}
\end{align*}
is non-decreasing with $l$.
Taking the limit $l \rightarrow \infty$ therefore preserves the inequality and we then get
$$|P_{\Upsilon,k}(\alpha)-P_{\Upsilon',k}(\alpha)| \leq 1 - \frac{n-1}{n} \cdot \frac{n-2}{n}\cdot \dots \cdot \frac{n-k+1}{n}.$$

Now, to prove the proposition also for Model B, we proceed as follows. It is easy to verify that $Q_{\Upsilon}(\alpha)$ is equivalently given by the following process. First, we uniformly sample a local example $\omega$ of width $k = |\textit{vars}(\alpha)|$ and then we uniformly sample a substitution $\vartheta$ from the set $\Theta_k$ of injective substitutions from $\textit{vars}(\alpha)$ to the set $\{1,2,\dots, k \}$. $Q_\Upsilon(\alpha)$ is then also equal to the probability that $\omega \models \alpha\vartheta$. We can then exploit the arguments used to prove the result for Model A. We have $\left|Q_{\Upsilon}(\alpha)-Q_{\Upsilon'}(\alpha)\right| =$
\begin{align*}
&=\left| \frac{1}{|\Theta_k|} \sum_{\vartheta \in \Theta_k}\sum_{\omega : \omega \models \alpha\vartheta} P_{\Upsilon,k}(\omega)-P_{\Upsilon',k}(\omega) \right| \\
&\leq 1-\left(\frac{n-k+1}{n}\right)^{k-1} \\
&= 1-\left(\frac{n-|\textit{vars}(\alpha)|+1}{n}\right)^{|\textit{vars}(\alpha)|-1}
\end{align*}
where the inequality follows from the same algebraic manipulations that we performed above for Model A.
\end{proof}

\noindent {\bf Proposition \ref{prop:interiority}.}
{\em Let $\theta$ be a vector representing the values of a set of relational marginals given by formulas from a set $\Phi = \{\alpha_1,\dots,\alpha_l \}$. Let $k$ be the width of the relational marginals of Model A or $k = \max_{\alpha_i} |\textit{vars}(\alpha_i)|$ for Model B. Let the set of hard rules $\Phi_0$ be empty. If $\theta$ is in the $\left(\eta+\sqrt{l} \left( 1-\left(\frac{m-k+1}{m}\right)^{k-1}\right)\right)$-interior of the relational marginal polytope of $\Phi$ of domain-size $m$ then it is also in the $\eta$-interior of the relational marginal polytope of $\Phi$ for any domain size $m'$.}
\begin{proof}
Let us denote by $S(\Upsilon_i)$ the vector of values of the relational marginals corresponding to the formulas $\alpha_1$, $\dots$, $\alpha_l$ computed from $\Upsilon_i$. Let $B$ be a point on the boundary of the relational marginal polytope. It follows from the definitions that we can write $B$ as a linear combination\footnote{In particular, this follows from the fact that relational marginal polytope is a convex hull of a finite set of points in $l$-dimensional space where the points are given by $S(\Upsilon)$ for $\Upsilon \in \Pi(\cC,\emptyset)$.} $B = a_1 \cdot S(\Upsilon_1) + \dots + a_{l} \cdot S(\Upsilon_{l})$ where all $a_i$'s are positive and $a_1+\dots+a_{l} = 1$ for some suitable $\Upsilon_1, \dots, \Upsilon_{l}$ on domain of size $m$. Let $\Upsilon_1', \dots, \Upsilon_{l}'$ be the expansions of the respective $\Upsilon_1,\dots,\Upsilon_{l}$ (we pick the level of these expansions sufficient to make the cardinality of their domains greater than the desired cardinality $m'$) and let $B' = a_1 \cdot S(\Upsilon_1') + \dots + a_{l} \cdot S(\Upsilon_{l}')$. For notational convenience, let us denote
$$\xi = 1-\left(\frac{n-k+1}{n}\right)^{k-1}.$$ With this notation we have that $\theta$ is in the $(\eta+\xi\sqrt{l})$-interior of the marginal polytope of domain of size $m$.
Then, using Proposition \ref{prop:difference}, we have (for both Model A and Model~B):
$$\|S(\Upsilon_i)-S(\Upsilon_i')\|_2 \leq \xi \sqrt{l}$$
where we note that $l$ is the number of formulas in $\Phi$. Using triangle inequality, it follows that
\begin{multline*}
    \|B-B'\|_2 =\\
=\| a_1 (S(\Upsilon_1)-S(\Upsilon_1'))+\dots+ a_{l}(S(\Upsilon_{l})-S(\Upsilon_{l}'))\|_2 \\
\leq a_1 \|  S(\Upsilon_1)-S(\Upsilon_1')\|_2+\dots+ a_{l} \|S(\Upsilon_{l})-S(\Upsilon_{l}')\|_2 \\
\leq (a_1+\dots+a_{l}) \cdot \xi \cdot \sqrt{l} = \xi \sqrt{l} 
\end{multline*}
Intuitively this means that any point on the boundary of the polytope for domain size $m$ will have a point on the boundary or inside the polytope for domain size $m'$ in distance not greater than $\xi \sqrt{l}$. Since the point $\theta$ is in the $(\eta+\xi\sqrt{l})$-interior of the polytope for domain size $m$, it follows that it must also be in the $\eta$-interior of the polytope for domain size $m'$ (for any $m'$), which is what we needed to show.
\end{proof}

Next we prove Proposition \ref{prop:main-estimation} using a series of lemmas. We first state and prove the lemmas for Model A and after that, separately, also their analogoues for Model B. 

\begin{lemma}[Model A Case]\label{lemma:crazy1}
Let $\aleph = (\cA_\aleph, \cC_\aleph)$, where $\cC_\aleph$ is a set of constants and $\cA_\aleph$ is a set of ground atoms involving only constants from $\cC_\aleph$ and let $0 \leq n \leq |\cC_\aleph|$ and $0 \leq k \leq n$ be integers. Let $\mathcal{C}_\Upsilon$ be sampled uniformly from all size-$n$ subsets of $\mathcal{C}_\aleph$ and let $\Upsilon = \aleph\langle \cC_\Upsilon \rangle$. Let $\mathbf{X} = (\cS_1,\cS_2,\dots,\cS_{\lfloor \frac{n}{k} \rfloor})$ be a vector of subsets of $\cC_\aleph$, each sampled uniformly and independently of the others from all size-$k$ subsets of $\cC_\aleph$. Next let $\mathbf{Y} = (\cS_1',\cS_2',\dots,\cS_{\lfloor \frac{n}{k} \rfloor}')$ be a vector sampled by the following process:
\begin{enumerate}
    \item Let $\mathcal{I}' = \{ 1,2,\dots,|\cC_\aleph| \}$.
    \item Sample subsets $\mathcal{I}_1',\dots,\mathcal{I}_{\lfloor \frac{n}{k} \rfloor}'$ of size $k$ uniformly from $\mathcal{I}'$.
    \item Sample an injective function $g : \bigcup_{\mathcal{I}_i' \subseteq \mathcal{I}'} \mathcal{I}_i' \rightarrow \cC_\Upsilon$ uniformly from all such functions.
    \item Let $\cS_i' = g(\mathcal{I}_i')$ for all $0 \leq i \leq \lfloor \frac{n}{k} \rfloor$.
\end{enumerate}
Then $\mathbf{X}$ and $\mathbf{Y}$ have the same distribution.
\end{lemma}
\begin{proof}
(Sketch) Let $h$ be a bijection from $\mathcal{I}$ to $\cC_\aleph$ sampled uniformly from the set of all such bijections. If we sample $\mathcal{I}_1,\dots,\mathcal{I}_{\lfloor \frac{n}{k} \rfloor}$ of size $k$ uniformly from $\mathcal{I}$ then the distribution of $(h(\mathcal{I}_1), \dots, h(\mathcal{I}_{\lfloor \frac{n}{k} \rfloor}))$ will be the same as that of $\mathbf{X}$. In fact, most of $h$ is actually irrelevant and we can replace sampling $h$ by uniformly sampling an injective function $g : \bigcup \mathcal{I}_i \rightarrow \cC_\aleph$ after we sample the sets $\mathcal{I}_i$. Finally, it is easy to realize that the function $g$ can also be sampled as follows. First, sample a subset $\cC$ of $\cC_\aleph$ of size $n$ uniformly and then sample uniformly an injective function $g$ such that $\operatorname{Range}(g) \subseteq \cC$. The reason why we can do this is because $\left|\bigcup \mathcal{I}_i \right| \leq n$. This finishes the proof of the lemma as this is actually the same as the process of sampling $\mathbf{Y}$.
\end{proof}

\begin{lemma}[Model A Case]\label{lemma:hoeffding1}
Let $\aleph = (\cA_\aleph, \cC_\aleph)$, where $\cC_\aleph$ is a set of constants and $\cA_\aleph$ is a set of ground atoms involving only constants from $\cC_\aleph$ and let $0 \leq n \leq |\cC_\aleph|$ and $0 \leq k \leq n$ be integers. Let $\mathcal{C}_\Upsilon$ be sampled uniformly from all size-$n$ subsets of $\mathcal{C}_\aleph$ and let $\Upsilon = \aleph\langle \cC_\Upsilon \rangle$. Let $\mathbf{Y}$ be sampled as in Lemma \ref{lemma:crazy1} (i.e.\ $\mathbf{Y}$ is sampled only using $\Upsilon$ and not directly $\aleph$). Let $\alpha$ be a closed and constant-free formula. Let 
$$\widehat{A}_\Upsilon = \frac{1}{\lfloor \frac{n}{k} \rfloor} \sum_{\mathcal{S}_i' \in \mathbf{Y}} \mathds{1}(\Upsilon\langle \mathcal{S}_i' \rangle \models \alpha)$$ 
and let $A_{\aleph} = P_{\aleph,k}(\alpha)$. Then we have
$$P\left[ \left| \widehat{A}_\Upsilon - A_\aleph \right| \geq \epsilon \right] \leq 2 \exp\left(-2 \left\lfloor \frac{n}{k} \right\rfloor \epsilon^2 \right).$$
\end{lemma}
\begin{proof}
We define an estimator on the set $\mathbf{X}$ which is sampled as in Lemma \ref{lemma:crazy1}
$$\widehat{A} = \frac{1}{\left\lfloor \frac{n}{k} \right\rfloor} \sum_{\mathcal{S}_i \in \mathbf{X}} \mathds{1}(\aleph\langle \mathcal{S}_i \rangle \models \alpha).$$
By Lemma \ref{lemma:crazy1}, we know that 
$$P\left[ \left| \widehat{A} - A_\aleph \right| \geq \epsilon \right] = P\left[ \left| \widehat{A}_\Upsilon - A_\aleph \right| \geq \epsilon \right],$$
so we only need to show 
$$P\left[ \left| \widehat{A} - A_\aleph \right| \geq \epsilon \right]\leq 2 \exp\left(-2 \left\lfloor \frac{n}{k} \right\rfloor \epsilon^2 \right).$$
Since $\cS_i$ in $\mathbf{X}$ are independent, the inequality above follows immediately from the fact that $\expect{\widehat{A}} = A_\aleph$ (which follows $\expect{\mathds{1}(\aleph\langle \mathcal{S}_i \rangle \models \alpha)} = A_\aleph$ and the linearity property of expectation) and the Chernoff-Hoeffding theorem.
\end{proof}

\begin{lemma}[Model A Case]\label{prop:expected_sampling}
Let $\aleph = (\cA_\aleph, \cC_\aleph)$, where $\cC_\aleph$ is a set of constants and $\cA_\aleph$ is a set of ground atoms involving only constants from $\cC_\aleph$ and let $0 \leq n \leq |\cC_\aleph|$ and $0 \leq k \leq n$ be integers. Let $\mathcal{C}_\Upsilon$ be sampled uniformly from all size-$n$ subsets of $\mathcal{C}_\aleph$ and let $\Upsilon = \aleph\langle \cC_\Upsilon \rangle$. Let $\mathbf{Y}$ be sampled as in Lemma \ref{lemma:crazy1} (i.e.\ $\mathbf{Y}$ is sampled only using $\Upsilon$ and not directly $\aleph$). Let $\alpha$ be a closed and constant-free formula. Let 
$$\widehat{A}_\Upsilon = \frac{1}{\lfloor \frac{n}{k} \rfloor} \sum_{\mathcal{S}_i' \in \mathbf{Y}} \mathds{1}(\Upsilon\langle \mathcal{S}_i' \rangle \models \alpha)$$ 
and let $A_{\aleph} = P_{\aleph,k}(\alpha)$. Then we have
$$\expect{\left| \widehat{A}_\Upsilon - A_\aleph \right|} \leq \sqrt{\frac{1+2\log 2}{4\lfloor n/k \rfloor}}.$$
\end{lemma}
\begin{proof}
Let $Z:=\left| \widehat{A}_\Upsilon - A_\aleph \right|.$
First, notice that 
$$\expect{Z}\le \sqrt{\expect{Z^2}},$$ 
so we now try to bound $\expect{Z^2}$. 
To bound $\expect{Z^2}$, note that for any $u>0$,
\begin{align*}
\expect{Z^2} &= \int_0^\infty P\left[Z^2 \geq t \right] dt\\
&= \int_0^\infty P\left[Z \geq \sqrt{t} \right] dt\\
&\leq u + \int_u^\infty P\left[Z \geq \sqrt{t} \right] dt\\
&\leq u + 2\int_u^\infty \exp\left(-2 \left\lfloor \frac{n}{k} \right\rfloor t \right)dt\\
&= u + \exp\left(-2\left\lfloor \frac{n}{k} \right\rfloor u \right)/\left\lfloor \frac{n}{k} \right\rfloor.
\end{align*}
Since $u$ was arbitrary, we may choose it to minimize the obtained bound. 
The optimal choice is $u=\frac{\log 2}{2\lfloor n/k \rfloor}$, which yields $\expect{Z^2} \leq \frac{1+2\log 2}{4\lfloor n/k \rfloor}$. 
\end{proof}

\begin{lemma}[Model A Case]\label{prop:u-expected}
Let $\aleph = (\cA_\aleph, \cC_\aleph)$, where $\cC_\aleph$ is a set of constants and $\cA_\aleph$ is a set of ground atoms involving only constants from $\cC_\aleph$ and let $0 \leq n \leq |\cC_\aleph|$ and $0 \leq k \leq n$ be integers. Let $\mathcal{C}_\Upsilon$ be sampled uniformly from all size-$n$ subsets of $\mathcal{C}_\aleph$ and let $\Upsilon = \aleph\langle \cC_\Upsilon \rangle$. Let $\alpha$ be a closed and constant-free formula and let $\cC_k$ denote all size-$k$ subsets of $\cC_\Upsilon$. Let 
$$\widetilde{A}_\Upsilon = \left( \begin{array}{c} n \\ k \end{array} \right)^{-1} \sum_{\cS \in \cC_k} \mathds{1}(\Upsilon\langle \cS \rangle \models \alpha)$$ 
and let $A_{\aleph} = P_{\aleph,k}(\alpha)$. Then we have
$$\expect{\left| \widetilde{A}_\Upsilon - A_\aleph \right|} \leq \sqrt{\frac{1+2\log 2}{4\lfloor n/k \rfloor}}.$$
\end{lemma}
\begin{proof}
First we define an auxiliary estimator $\widetilde{A}^{(q)}_{\Upsilon}$. Let $\mathbf{Y}^{(q)}$ be a vector of $\lfloor n/k \rfloor \cdot q$ size-$k$ subsets of $\cC_\Upsilon$ where the subsets of $\cC_{\Upsilon}$ in each of the $q$ non-overlapping size-$\lfloor n/k \rfloor$ segments $\mathbf{Y}_1^{(q)}, \mathbf{Y}_2^{(q)},\dots,\mathbf{Y}_q^{(q)}$ of $\mathbf{Y}^{(q)}$ are sampled in the same way as the elements of the vector $\mathbf{Y}$ in Lemma \ref{lemma:crazy1} (all with the same $\cC_\Upsilon$). Let us define 
$$\widetilde{A}^{(q)}_{\Upsilon} = \frac{1}{q \cdot \lfloor n/k \rfloor} \sum_{\cS \in \mathbf{Y}^{(q)}} \mathds{1}(\Upsilon\langle \cS \rangle \models \alpha). $$
Using Lemma \ref{prop:expected_sampling}, we get
\begin{align*}
    & \expect{\left| \widetilde{A}_\Upsilon^{(q)} - A_\aleph \right|} = \\
    &= \expect{\left| \left( \frac{1}{q \cdot \lfloor n/k \rfloor} \sum_{\cS \in \mathbf{Y}^{(q)}} \mathds{1}(\Upsilon\langle \cS \rangle \models \alpha)\right) -A_\aleph \right|} \\
    &= \expect{\left| \left( \frac{1}{q} \sum_{i=1}^{q} \frac{1}{\lfloor n/k \rfloor} \sum_{\cS \in \mathbf{Y}_{i}^{(q)}} \mathds{1}(\Upsilon\langle \cS \rangle \models \alpha)\right) -A_\aleph \right|} \\
    &= \frac{1}{q} \expect{\left|\sum_{i=1}^{q} \left( \frac{1}{\lfloor n/k \rfloor} \sum_{\cS \in \mathbf{Y}_{i}^{(q)}} \mathds{1}(\Upsilon\langle \cS \rangle \models \alpha) -A_\aleph\right)  \right|} \\
    &\leq \frac{1}{q} \sum_{i=1}^{q} \expect{\left| \left( \frac{1}{\lfloor n/k \rfloor} \sum_{\cS \in \mathbf{Y}_{i}^{(q)}} \mathds{1}(\Upsilon\langle \cS \rangle \models \alpha) -A_\aleph\right)  \right|} \\
    &\leq q \cdot \frac{1}{q} \sqrt{\frac{1+2\log 2}{4\lfloor n/k \rfloor}} = \sqrt{\frac{1+2\log 2}{4\lfloor n/k \rfloor}}
\end{align*}

\noindent We have (again using triangle inequality)
\begin{align*}
    \expect{ \left| \widetilde{A}_\Upsilon-A_\aleph \right|} = \expect{ \left| \widetilde{A}_\Upsilon - \widetilde{A}^{(q)}_{\Upsilon} + \widetilde{A}^{(q)}_{\Upsilon} -A_\aleph \right|} \\
    \leq \expect{ \left| \widetilde{A}_\Upsilon - \widetilde{A}^{(q)}_{\Upsilon} \right|} + \expect{\left| \widetilde{A}^{(q)}_{\Upsilon} -A_\aleph \right|} \\
    \leq \expect{ \left| \widetilde{A}_\Upsilon - \widetilde{A}^{(q)}_{\Upsilon} \right|} + \sqrt{\frac{1+2\log 2}{4\lfloor n/k \rfloor}}.
\end{align*}

\noindent It follows from the strong law of large numbers (which holds for any $\Upsilon$) and Proposition \ref{prop:unbiased_estimate} that $P[\lim_{q \rightarrow \infty}\widetilde{A}^{(q)}_{\Upsilon} = \widetilde{A}_\Upsilon] = 1$. Since $q$ was arbitrary, we can use this to conclude
$$\expect{ \left| \widetilde{A}_{\Upsilon} -A_\aleph \right| } \leq \sqrt{\frac{1+2\log 2}{4\lfloor n/k \rfloor}}.$$
\end{proof}

Next we state the analogues of the above lemmas also for Model B (for the first reading, we recommend skipping these lemmas and proceeding directly to the proof of Proposition~\ref{prop:main-estimation}). We omit the proofs as they are completely analogical to their respective counterparts for Model A (we just replace sampling of subsets by sampling of injective substitutions).

\begin{lemma}[Model B Case]\label{lemma:crazy2}
Let $\aleph = (\cA_\aleph, \cC_\aleph)$, where $\cC_\aleph$ is a set of constants and $\cA_\aleph$ is a set of ground atoms involving only constants from $\cC_\aleph$ and let $0 \leq n \leq |\cC_\aleph|$ and $0 \leq k \leq n$ be integers. Let $\mathcal{C}_\Upsilon$ be sampled uniformly from all size-$n$ subsets of $\mathcal{C}_\aleph$ and let $\Upsilon = \aleph\langle \cC_\Upsilon \rangle$. Let $\mathbf{X} = (\vartheta_1,\vartheta_2,\dots,\vartheta_{\lfloor \frac{n}{k} \rfloor})$ be a vector of injective substitutions from a given size-$k$ set of variables $\mathcal{V}$ to $\cC_\aleph$, each sampled uniformly and independently of the others. Next let $\mathbf{Y} = (\vartheta_1',\vartheta_2',\dots,\vartheta_{\lfloor \frac{n}{k} \rfloor}')$ be a vector sampled by the following process:
\begin{enumerate}
    \item Let $\mathcal{I}' = \{ 1,2,\dots,|\cC_\aleph| \}$.
    \item Sample {\bf ordered} subsets $\mathcal{I}_1',\dots,\mathcal{I}_{\lfloor \frac{n}{k} \rfloor}'$ of size $k$ uniformly from $\mathcal{I}'$.
    \item Sample an injective function $g : \bigcup_{\mathcal{I}_i' \subseteq \mathcal{I}'} \mathcal{I}_i' \rightarrow \cC_\Upsilon$ uniformly from all such functions.
    \item Let $\vartheta_i' = \{ V_j \mapsto c_j | V_j \in \mathcal{V} \mbox{ and } (c_1, \dots, c_{|\mathcal{V}|}) = g(\mathcal{I}_i') \}$ for all $0 \leq i \leq \lfloor \frac{n}{k} \rfloor$, i.e. $\vartheta'_i$ is an injective substitution from $\mathcal{V}$ to $\cC_{\Upsilon}$.
\end{enumerate}
Then $\mathbf{X}$ and $\mathbf{Y}$ have the same distribution.
\end{lemma}

\begin{lemma}[Model B Case]\label{lemma:hoeffding2}
Let $\aleph = (\cA_\aleph, \cC_\aleph)$, where $\cC_\aleph$ is a set of constants and $\cA_\aleph$ is a set of ground atoms involving only constants from $\cC_\aleph$ and let $0 \leq n \leq |\cC_\aleph|$ and $0 \leq k \leq n$ be integers. Let $\mathcal{C}_\Upsilon$ be sampled uniformly from all size-$n$ subsets of $\mathcal{C}_\aleph$ and let $\Upsilon = \aleph\langle \cC_\Upsilon \rangle$. Let $\alpha$ be a closed and constant-free formula. Let $\mathbf{Y}$ be sampled as in Lemma \ref{lemma:crazy2} (i.e.\ $\mathbf{Y}$ is sampled only using $\Upsilon$ and not directly $\aleph$), where we set $\mathcal{V} = \textit{vars}(\alpha)$. Let 
$$\widehat{A}_\Upsilon = \frac{1}{\lfloor \frac{n}{k} \rfloor} \sum_{\vartheta_i' \in \mathbf{Y}} \mathds{1}(\Upsilon \models \alpha\vartheta_i')$$ 
and let $A_{\aleph} = Q_{\aleph}(\alpha)$. Then we have
$$P\left[ \left| \widehat{A}_\Upsilon - A_\aleph \right| \geq \epsilon \right] \leq 2 \exp\left(-2 \left\lfloor \frac{n}{k} \right\rfloor \epsilon^2 \right).$$
\end{lemma}

\begin{lemma}[Model B Case]\label{prop:expected_sampling2}
Let $\aleph = (\cA_\aleph, \cC_\aleph)$, where $\cC_\aleph$ is a set of constants and $\cA_\aleph$ is a set of ground atoms involving only constants from $\cC_\aleph$ and let $0 \leq n \leq |\cC_\aleph|$ and $0 \leq k \leq n$ be integers. Let $\mathcal{C}_\Upsilon$ be sampled uniformly from all size-$n$ subsets of $\mathcal{C}_\aleph$ and let $\Upsilon = \aleph\langle \cC_\Upsilon \rangle$. Let $\mathbf{Y}$ be sampled as in Lemma \ref{lemma:crazy2} (i.e.\ $\mathbf{Y}$ is sampled only using $\Upsilon$ and not directly $\aleph$), where we set $\mathcal{V} = \textit{vars}(\alpha)$. Let $\alpha$ be a closed and constant-free formula. Let 
$$\widehat{A}_\Upsilon = \frac{1}{\lfloor \frac{n}{k} \rfloor} \sum_{\vartheta_i' \in \mathbf{Y}} \mathds{1}(\Upsilon \models \alpha\vartheta_i')$$ 
and let $A_{\aleph} = Q_{\aleph}(\alpha)$. Then we have
$$\expect{\left| \widehat{A}_\Upsilon - A_\aleph \right|} \leq \sqrt{\frac{1+2\log 2}{4\lfloor n/k \rfloor}}.$$
\end{lemma}

\begin{lemma}[Model B Case]\label{prop:u-expected2}
Let $\aleph = (\cA_\aleph, \cC_\aleph)$, where $\cC_\aleph$ is a set of constants and $\cA_\aleph$ is a set of ground atoms involving only constants from $\cC_\aleph$ and let $0 \leq n \leq |\cC_\aleph|$ and $0 \leq k \leq n$ be integers. Let $\mathcal{C}_\Upsilon$ be sampled uniformly from all size-$n$ subsets of $\mathcal{C}_\aleph$ and let $\Upsilon = \aleph\langle \cC_\Upsilon \rangle$. Let $\alpha$ be a closed and constant-free formula and $\Theta_\alpha$ denote all injective substitutions from $\textit{vars}(\alpha)$ to $\cC_\Upsilon$. Let 
$$\widetilde{A}_\Upsilon = \frac{1}{|\Theta_\alpha|} \sum_{\vartheta \in \Theta_\alpha} \mathds{1}(\Upsilon \models \alpha\vartheta)$$ 
and let $A_{\aleph} = Q_{\aleph}(\alpha)$. Then we have
$$\expect{\left| \widetilde{A}_\Upsilon - A_\aleph \right|} \leq \sqrt{\frac{1+2\log 2}{4\lfloor n/|\textit{vars}(\alpha)| \rfloor}}.$$
\end{lemma}

Now we have all the ingredients needed to prove Proposition \ref{prop:main-estimation}.

\noindent {\bf Proposition \ref{prop:main-estimation}}
{\em Let $m$ and $n$ be positive integers, $\alpha$  a closed formula and let $k$ be the width of local examples. Let $\aleph = (\cA_\aleph,\cC_\aleph)$ be a global example, $\cC_\Upsilon$ be sampled uniformly among all size-$m$ subsets of $\cC_\aleph$ and $\Upsilon = \aleph\langle \cC_\Upsilon \rangle$. Let $\widehat{A}_{\aleph} = P_{\aleph,k}(\alpha)$. Let $\widehat{B}_\Upsilon$ be an estimate computed from the $l$-level expansion of $\cC_\Upsilon$. Then 
$$\expect{\left| \widehat{A}_\aleph - \widehat{B}_{\Upsilon} \right|} \leq 1-\left(\frac{m-k+1}{m}\right)^{k-1} + \sqrt{\frac{1+2\log 2}{4\lfloor m/k \rfloor}}.$$
In the case of model B, the same upper bound holds if we choose $k = |\textit{vars}(\alpha)|$ and $\widehat{A}_{\aleph} = Q_{\aleph}(\alpha)$.}
\begin{proof}
Let $\widehat{A}_\Upsilon = P_{\Upsilon,k}(\alpha)$ ($\widehat{A}_\Upsilon = Q_{\Upsilon}(\alpha)$, respectively). Then we have
\begin{align*}
    \expect{\left| \widehat{A}_\aleph - \widehat{B}_{\Upsilon} \right|} = \expect{\left| \widehat{A}_\aleph - \widehat{A}_\Upsilon + \widehat{A}_\Upsilon - \widehat{B}_{\Upsilon} \right|} \\
    \leq  \expect{\left|\widehat{A}_\Upsilon - \widehat{B}_{\Upsilon} \right|} + \expect{\left| \widehat{A}_\aleph - \widehat{A}_\Upsilon \right|} \\
    \leq 1-\left(\frac{m-k+1}{m}\right)^{k-1} + \sqrt{\frac{1+2\log 2}{4\lfloor m/k \rfloor}}
\end{align*}
where the last inequality follows from Proposition \ref{prop:difference} and Lemma \ref{prop:u-expected} for Model A, and Lemma \ref{prop:u-expected2} for Model B, respectively.
\end{proof}

\fi


\bibliographystyle{aaai}
\bibliography{bibliography}

\end{document}